%% file: main_snp.tex
\documentclass[conference,compsoc]{IEEEtran}

\ifCLASSOPTIONcompsoc
  \usepackage[nocompress]{cite}
\else
  \usepackage{cite}
\fi

\ifCLASSINFOpdf
\else
\fi

\ifCLASSOPTIONcompsoc
 \usepackage[caption=false,font=footnotesize,labelfont=sf,textfont=sf,position=top]{subfig}
\else
 \usepackage[caption=false,font=footnotesize]{subfig}
\fi

\usepackage{fixltx2e}

\usepackage{stfloats}

\usepackage{url}

\input{math_commands.tex}

\usepackage{booktabs}
\usepackage{subcaption}
\usepackage{url}
\usepackage{amsmath,amsthm,amsfonts}
\usepackage{tabularx}
\usepackage{xcolor}
\usepackage[short,nocomma]{optidef}
\usepackage{hyperref}
\usepackage{todonotes}
\usepackage{thmtools, thm-restate}
\usepackage{multirow}
\usepackage{multicol}
\usepackage{dsfont}
\usepackage{xspace}
\usepackage{etoolbox}
\usepackage{float}
\usepackage{algorithm}
\usepackage{algpseudocode}
\usepackage{enumitem}

\AfterEndEnvironment{tabular}{\vspace{1em}}
\makeatletter
\newtheoremstyle{mytheorem}%
  {3pt}%
  {3pt}%
  {\itshape}%
  {}%
  {\scshape}%
  {.}%
  {.5em}%
  {\thmname{#1}\thmnumber{\@ifnotempty{#1}{ }#2}%
   \thmnote{ {\the\thm@notefont(#3)}}}%
\makeatother
\theoremstyle{mytheorem}

\newcommand{\ztitle}{PEANUT: Perturbations by Eigenvector Alignment for Attacking Graph Neural Networks Under Topology-Driven Message Passing}

\newcommand{\MethodName}{\text{PEA}\xspace}  %
\newcommand{\MethodDesc}{\emph{\underline{P}erturbation by \underline{E}igenvector \underline{A}lignment}\xspace}

\newcommand{\norm}[1]{\left\lVert#1\right\rVert}
\newcommand{\tr}{\rm{tr}}

\def\bLambda{{\boldsymbol{\Lambda}}}
\def\bTheta{{\boldsymbol{\Theta}}}
\def\bDelta{{\boldsymbol{\Delta}}}

\def\nv{{n_v}}

\newcommand{\MLP}{\text{MLP}\xspace}
\newcommand{\SGC}{\text{SGC}\xspace}
\newcommand{\GIN}{\text{GIN}\xspace}
\newcommand{\GCN}{\text{GCN}\xspace}
\newcommand{\SAGE}{\text{SAGE}\xspace}
\newcommand{\rand}{\text{Rand}\xspace}

\renewcommand{\paragraph}{\subsubsection*}
\newcommand{\citep}{\cite}
\newcommand{\citet}{\cite}

\begin{document}
\title{\ztitle}

\IEEEoverridecommandlockouts

\author{
\IEEEauthorblockN{Bhavya Kohli\thanks{This work is a preprint}, Biplab Sikdar}
\IEEEauthorblockA{National University of Singapore, Singapore\\bhavyakohli@u.nus.edu\\ bsikdar@nus.edu.sg}
}

\maketitle

\begin{abstract}
Message Passing Neural Networks (MPNNs) have achieved strong performance on tasks involving relational data. However, small perturbations to graph structure can significantly alter their outputs, raising concerns about their robustness in real-world deployment in security critical environments. 
In this work, we study a core vulnerability in MPNNs that explicitly consume graph topology via the adjacency matrix or Laplacian as part of their message passing mechanism. 
We show that this design choice exposes an extremely vulnerable attack surface, with significant effects even under minimal perturbation. 
Building on this observation, we propose PEA, a simple, gradient-free, black-box injection attack that requires only a single query to the target model, capitalizing on this vulnerability by constructing a perturbation aligned with a specific significant eigenvector to induce large output deviations. 
Unlike graph modification attacks, PEA operates under the realistic assumption that adversaries cannot alter the original graph and are limited to injecting new nodes at inference time.
PEA requires no iterative optimization, parameter learning, or surrogate models---which require additional training and remain susceptible to differences in model priors and generalization capabilities---thereby avoiding significant computational overhead and the associated transferability challenges. 
We evaluate PEA on popular benchmark datasets across three graph learning tasks, showing consistent performance degradation under realistic attack constraints despite its simplicity. 
Our results reveal a fundamental security weakness in topology-driven message passing architectures and urge an implementation shift, as the worst effects of such attacks can be substantially mitigated through appropriate input filtering.

\end{abstract}

\IEEEpeerreviewmaketitle

\section{Introduction}\label{sec:intro}
Message Passing Neural Networks (MPNNs) have emerged as a central modeling paradigm for learning over relational and structured data across a wide range of applications, including molecular property prediction~\citep{gilmer2017neural,moleculenet}, physical simulation~\citep{sanchez2020learning}, traffic forecasting~\citep{jiang2022graph,li2017diffusion}, recommendation systems~\citep{ying2018graph}, and knowledge graph reasoning~\citep{schlichtkrull2018modeling}. As they are increasingly deployed in high-impact settings, their adversarial robustness is a practical concern. 
Their success largely stems from the message-passing paradigm, where node representations are iteratively updated by aggregating information from local neighborhoods. 
However, this same mechanism also exposes MPNNs to structural vulnerabilities: small, carefully crafted changes in the graph can propagate through the network and lead to significant performance degradation. Early studies on adversarial attacks against MPNNs primarily focus on graph-modification attacks, where the adversary perturbs existing edges or node features~\citep{zugner2020adversarial,zugner2019metattack,dai2018adversarial}. While effective in controlled settings, these attacks often rely on unrealistic assumptions about the attacker's capability, such as having direct access to an existing graph’s topology and being able to modify its attributes. 
In many real-world systems (e.g., social networks, citation graphs, recommendation systems), modifying existing users, items, or relationships is generally infeasible. 
In contrast, graph injection attacks provide a more practical threat model where the adversary introduces new nodes into the graph and connects them to existing nodes without altering the original graph structure~\citep{ju2023let,zou2021tdgia,wang2020scalable,sun2020adversarial}. 
Such attacks reflect more realistic scenarios, such as creating fake accounts or synthetic entities to influence downstream performance. %
In this work, we focus on the evasion setting, wherein the attacker can inject new nodes to manipulate predictions only at inference time, and does not have the capability to poison the graph at the training phase \citep{ju2023let,zou2021tdgia,dai2018adversarial}. This scenario is particularly relevant when retraining is costly or infrequent, and where adversaries can only interact with the system by introducing new entities at inference.

\paragraph{Vulnerability of adjacency-input architectures}
Many widely used MPNN architectures explicitly utilize the adjacency (or Laplacian) matrix \citep{fu2022p,wu2019simplifying,kipf2016semi} as message-passing operators. This design choice introduces a major architectural vulnerability: perturbations to the adjacency matrix directly alter the propagation pathways through which information flows, affecting the graph's receptive fields, aggregation neighborhoods, and spectral properties. This ultimately results in amplified, nonlocal changes in downstream representation. 
Unlike feature perturbations, which are often normalized or regularized, such adjacency perturbations modify the structure of the computation graph itself. %

Despite recent progress on graph injection attacks, most existing methods remain limited in both practicality and scope. A large fraction of prior work relies on per-graph iterative optimization or reinforcement learning to construct the injected nodes’ features and edges, which is computationally expensive and must be repeated for every target graph or attack instance \citep{liu2025query,tao2021gnia,wang2020scalable,sun2020adversarial}. Moreover, current injection attacks are predominantly evaluated on node-level tasks such as node classification, with limited evidence of effectiveness beyond this narrow setting \citep{ju2023let,tao2021gnia,sun2020adversarial}. 
Several recent approaches approximate the victim model by training surrogate models, which may be susceptible to having different priors and generalization capabilities, while also introducing additional time and resource overhead for their training. 
This makes the attack pipeline even more cumbersome in practice and limits its applicability in large-scale or real-time settings. %

\paragraph{Contributions}
To address the aforementioned drawbacks of previous works, we propose a black-box attack, namely 
\MethodDesc (\MethodName),
which works on the evasion setting. Our proposed method can be applied \emph{immediately} to a trained MPNN without any per-graph optimization or surrogate model training, only requiring the attacker to view the final node-level representations. Furthermore, it generalizes beyond node classification and can be directly applied to MPNNs trained for arbitrary graph-level objectives, including the largely underexplored graph regression setting. We summarize our contributions as follows:
\begin{itemize}[leftmargin=8mm]
    \item We identify a restricted black-box attack surface induced by MPNN architectures that explicitly utilize graph topology matrices for message passing.
    \item We propose a simple, gradient and surrogate-free, restricted black-box attack that injects virtual nodes to maximize differences between clean and perturbed graph representations.
    \item We provide the first known systematic evaluation of such attacks on graph-level regression tasks, demonstrating that even small perturbations can meaningfully degrade performance on graph regression benchmarks.
    \item We evaluate this attack surface over four graph tasks spanning both node-level and graph-level objectives across multiple MPNN architectures, each requiring minimal to no modifications in the attack pipeline, demonstrating the ease of exploiting this vulnerability.
\end{itemize}
\begin{figure*}
    \centering
    \includegraphics[width=\linewidth]{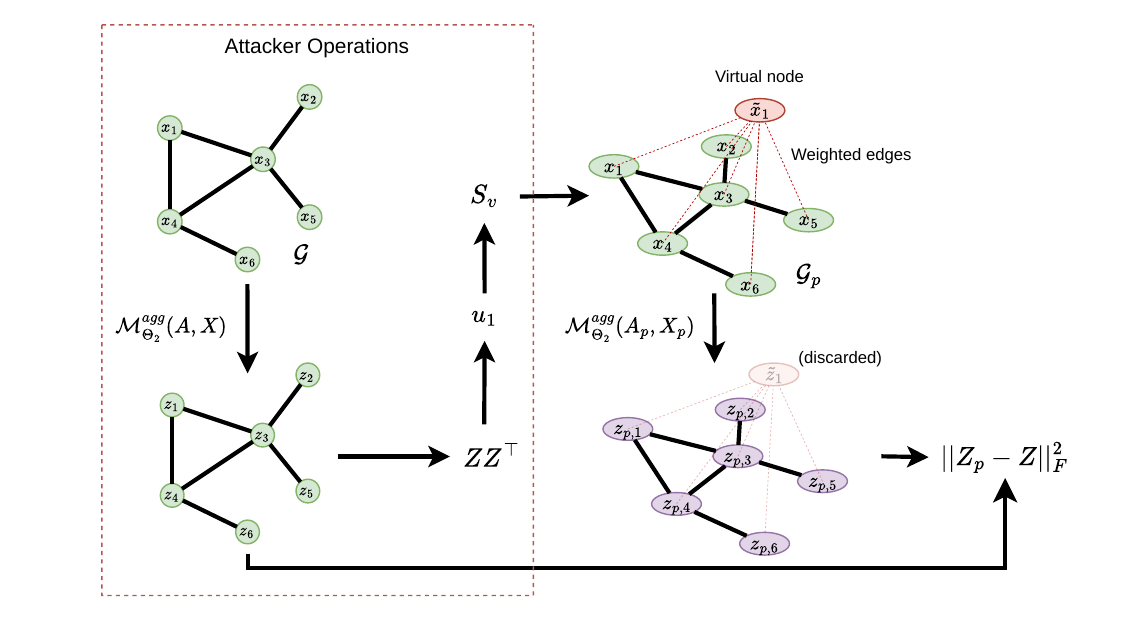}
    \caption{\MethodName when used on a graph-level task. The attacker queries the model once to obtain the (clean) node-level embeddings, using them to create the perturbation $\rmS_v$ using the eigenvector $\rvu_1$ of $\rmZ\rmZ\tran$ which induces a high $\gL(\rmS_v)$.}
    \label{fig:method}
\end{figure*} %

\subsection{Threat Model}
We consider a \emph{restricted black-box} threat model that works at the inference phase (evasion attack). 
The attacker has no access to model parameters, gradients, or training data but can observe only the final node-level representations during inference (e.g., output logits or log-probabilities for a node classifier, node representations prior to graph-pooling operations for graph-level tasks). 
This setting reflects modular deployment pipelines in which node-level embeddings are utilized by downstream components, while the underlying components themselves remain black boxes. 
The attacker is allowed to perturb the input graph structure via node injection (``virtual nodes") subject to a perturbation budget. 
The goal of the attacker is to induce large deviations in the generated embeddings by constructing perturbed graphs that maximize a discrepancy measure between clean and attacked graph representations.

\section{Related works}\label{sec:related}
\paragraph{Structural Attacks on MPNNs}
Adversarial attacks on MPNNs frequently target the graph structure by modifying the adjacency matrix through edge additions or deletions. Early white-box methods assume full access to model parameters and gradients; e.g., Nettack \citep{zugner2018adversarial} performs targeted perturbations of edges and features to induce misclassifications under structural constraints. Meta-learning-based poisoning attacks optimize the adjacency in a bi-level formulation to maximally degrade downstream performance~\citep{zugner2019metattack}. Gradient-based topology attacks, such as Fast Gradient Attack  (FGA)~\citep{chen2018fast}, relax discrete edges into continuous variables and greedily select perturbations via gradients. 
These approaches demonstrate that small structural changes can significantly impact predictions made by these models but rely on strong white-box assumptions and work on the graph modification that are unrealistic in many deployment settings. 
On the other hand, black-box, or \emph{restricted} black-box approaches such as GF-Attack~\citep{chang2020gfattack}, work under more feasible attack assumptions where attackers do not observe model parameters, gradients, or training data but may assume access to restricted internal representations or confidence scores.
Alternative black-box settings with no model queries have also been explored by optimizing surrogate spectral objectives that aim to maximally alter the implicit graph filters induced by MPNNs~\citep{xuxuan2020queryfree}. A few more classical white-box and recent black-box methods are summarized in~\citep{sun2022adversarialsurvey,jin2021adversarialsurvey,chen2020survey}.

\paragraph{Graph Injection Attacks}
A growing body of work studies graph injection attacks, where adversaries introduce malicious nodes to degrade model performance without modifying existing graph structure. The Node Injection Poisoning Attack (NIPA) \citep{sun2020adversarial} injects nodes during training using deep reinforcement learning to sequentially optimize edges and labels under a white-box setting. Due to its high compute requirements, it is considered not scalable for large-scale datasets \citep{zou2021tdgia}. 
Under the black-box setting, Topological Defective Graph Injection Attack (TDGIA) \citep{zou2021tdgia} offers a scalable inference-time injection attack that relies on iterative optimization of injected node features and connectivity; however, it relies on training surrogate models for its attack mechanism.
To maintain the unnoticeability of injection attacks, AGIA \citet{chen2022understanding} was proposed as a differentiable realization of the graph homophily constraint, showing that constraining injected nodes to match local homophily improves unnoticeability while maintaining their respective attack success. 
An extremely limited scenario of a single-node injection attack is considered in \citet{tao2021gnia}, wherein the objective is to find the most damaging injected node by a carefully crafted optimization scheme. This is further extended as the Generalizable Node Injection Attack model (G-NIA), which can work in black-box settings to amortize the per-graph optimization cost, although it requires a generator model to be trained. 
The Gradient-free Graph Advantage Actor Critic model (G2A2C) \citep{ju2023let} uses reinforcement learning to discover injection policies through repeated interaction with the target model, incurring substantial query and training overhead in the process.
QUGIA \citep{liu2025query}, a Query-based and Unnoticeable Graph Injection
Attack, tries to avoid the need to train surrogate models; however, it still requires iterative sampling and distribution updates per target graph. 
Graph injection attacks have also been extended to graph-level classification under a hard-label black-box setting \citet{zhou2023hard}, building upon existing edge perturbing attacks, using query-based optimization to inject nodes that flip graph predictions. %

\paragraph{Graph-level tasks and regression}
Most prior work on structural attacks focuses on node classification, with some exploring the effects on graph-level classification. Graph-level regression tasks, such as molecular property prediction or physical system modeling, remain comparatively underexplored in adversarial literature. 
In contrast to prior black-box approaches that rely on learning-based policies or query-driven optimization, our method leverages final node-level representations to construct gradient-free structural perturbations. 
We demonstrate that maximizing the norm discrepancy between clean and perturbed graphs constitutes an effective attack objective for graph-level regression, revealing a previously unexplored vulnerability of architectures that explicitly consume graph topology matrices such as the adjacency matrix or graph Laplacian as inputs.
Furthermore, we empirically demonstrate that this sort of norm discrepancy also translates to drops in performance in classification-based tasks. %

\section{Preliminaries}\label{sec:prelim}
Let $\gG(\gV,\rmA,\rmX)$ be an undirected attributed graph with vertex set $\gV$ of size $N$, and adjacency matrix $\rmA\in\sR^{N\times N}$ with associated edge set $\gE$. 
We denote the node attributes of $\gG$ as $\rmX\in\sR^{N\times D}$, and following the notation in \citet{wu2019simplifying}, we denote the normalized adjacency matrix as $\rmS=\rmD^{-1/2}\rmA\rmD^{-1/2}$, where $\rmD$ is the diagonal degree matrix of $\rmA$, and $\rmD_{ii}=\sum_j\rmA_{ij}$. 

\paragraph{Message Passing Neural Networks}
We consider a message passing neural network (MPNN) with parameters $\bTheta$, denoted $\gM_\bTheta(\rmA, \rmX)$, which uses the graph adjacency matrix and node attributes as inputs to produce task-specific outputs. %
For node-level tasks, such as node classification with $\gC$ classes, the network $\gM_\bTheta$ outputs log-probabilities (or unnormalized logits) $\rmZ\in\sR^{N\times \gC}$. For graph-level tasks, we represent the model as a composition 
\begin{equation}
    \gM_\bTheta(A,X) = \gM^{out}_{\bTheta_1}({\textsc{Pool}} (\gM^{agg}_{\bTheta_2}(A,X))) \;,
\end{equation}
where $\gM^{agg}_{\bTheta_2}$ first computes node-level embeddings $\rmZ\in\sR^{N\times d}$ which are then pooled to produce a single graph-level representation. This pooled embedding is finally reduced by the readout network $\gM^{out}_{\bTheta_1}$ to either a vector of $\gC$ values (graph-level classification with $\gC$ classes), or a single value (graph-level regression). 
This follows the design flow commonly used in molecular and physical system applications~\citep{gilmer2017neural,moleculenet}.

\paragraph{Virtual-node perturbation} 
For a graph $\gG$, we define a \emph{virtual-node perturbation} as the adjacency matrix with $\nv$ injected virtual nodes with the real-virtual, virtual-real, and virtual-virtual node connections defined by matrices $\rmA_{r,v}$, $\rmA_{v,r}$, and $\rmA_{v,v}$ respectively. 
The adjacency matrix of the perturbed graph with $N_p=N+\nv$ nodes is defined as $\rmA_{p}\in\sR^{N_p\times N_p}$. 
To minimize the number of new edges added, we consider $\rmA_{v,v}=\boldsymbol{0}^{\nv\times\nv}$, and to maintain symmetry, we consider $\rmA_{r,v} = \rmA_{v,r}\tran = \rmA_v \in \sR^{N\times\nv}$. 
Thus, a perturbation is uniquely determined by $\rmA_v$. If we work with the normalized adjacency matrix $\rmS$, the perturbed matrix follows the same structure, with $\rmS_{v}$ instead of $\rmA_{v}$.
The perturbed matrix is thus defined as:
\begin{equation}\label{def:perturbed-adj}
    \rmA_{p} = 
    \begin{pmatrix}
        \rmA & \rmA_{v} \\
        \rmA_{v}\tran & \boldsymbol{0}
    \end{pmatrix}
\end{equation}

The feature matrix of the perturbed graph is denoted by 
$\rmX_{p} = \begin{pmatrix}\rmX&\tilde{\rmX}\end{pmatrix}\tran\in\sR^{N_p\times D}$. 
To simplify the closed-form expression in Theorem \ref{thm:theorem1}, we assign a vector of zeros to all virtual nodes by default.  We refer to the perturbed graph as $\gG_p$. In Section \ref{sec:results_nc_xtilde}, we also explore alternative ways of choosing $\tilde{\rmX}$, and its effects on the attack.

\paragraph{Simplified Graph Convolution Network (SGC)} 
The simplified linear version of a Graph Convolution Network has been shown to perform similarly, if not better, than the \GCN \citep{kipf2016semi} at tasks such as node classification and text classification \citep{wu2019simplifying}. 
Given a 2-layer \SGC, we denote $\rmZ$ as the node embeddings for an unperturbed graph, i.e., $\rmZ=\rmS^2\rmX\bTheta$. Additionally, we define the node embeddings for the real nodes of a perturbed graph as $\rmZ_{p}$, which is $\rmS_p^2\rmX_p\bTheta[:N,:]$, i.e., the first $N$ rows of $\rmS_p^2\rmX_p\bTheta$. 
The SGC represents the most degenerate linearized version of an MPNN.

\paragraph{Attack Efficacy and Constraints} 
We define the \emph{attack efficacy} (denoted by $\gL(\rmS_v)$) of perturbation $\rmS_v$ as the norm of the difference of node embeddings of the real nodes for the original and perturbed graphs, i.e., $\gL(\rmS_v)=\norm{\rmZ_{p}-\rmZ}_F^2$. To constrain the perturbation that maximizes $\gL$, we impose a norm constraint $\norm{\rmS_v}_F^2\leq \bDelta^2$. 
In the case where we perturb the adjacency matrix, this norm constraint effectively bounds the exact number of new edges being added if the perturbation ($\rmA_v$ in this case) was binary. 
Our main optimization objective, therefore, is as follows:
\begin{equation}\label{eqn:main_objective}
\begin{split}
    \max_{\rmS_v\in\sR^{N\times \nv}}\quad&\gL(\rmS_v)=\norm{\rmZ_{p}-\rmZ}_F^2\\
    \rm{s.t.}\qquad&\norm{\rmS_v}_F^2\leq \bDelta^2
\end{split}
\end{equation}

For an \SGC network, $\gL(\rmS_v)$ reduces to the following, assuming $\tilde{\rmX}=0$ (see Appendix \ref{appendix:sgc_equivalence} for the derivation):
\begin{equation}\label{eqn:sgc_attack}
    \gL(\rmS_v)=\norm{\rmZ_{p}-\rmZ}_F^2=\norm{\rmS_{v}\rmS_{v}\tran\rmX\bTheta}_F^2
\end{equation}

\paragraph{Attack Goal} 
For the class of MPNNs defined above, admitting a form of the adjacency matrix (e.g., $\rmA$, $\rmS$, or graph Laplacian $\rmL$, etc.) as the input introduces a structural vulnerability. 
Our goal is to inject \emph{virtual nodes} with carefully chosen edge weights to maximally modify the final node-level embeddings $\rmZ$, and thereby degrade performance on downstream tasks.
The attack is \emph{gradient-free} and requires no access to model parameters $\bTheta$. 
For node-level tasks, this makes the attack strictly black-box; for graph-level tasks, it relies on observing the intermediate output of $\gM^{agg}_{\bTheta_2}$, making it a \emph{restricted black-box} method overall.
As seen in our experiments in Section~\ref{sec:experiments}, this approach is effective across various commonly used architectures, demonstrating that even limited adjacency perturbations can induce substantial changes in graph-level outputs. 

\section{Methodology}\label{sec:method}
In this section, we define the base version of our method, which we will refer to as \textbf{\MethodName}. 
First, we define a white-box variant for the \SGC architecture in Theorem \ref{thm:theorem1} which requires knowledge of the model parameters $\bTheta$. We refer to this variant as \MethodName-W. 
Then, in Section \ref{sec:practical}, we propose the black-box formulation, which also significantly deteriorates MPNN performance. We follow this with a discussion on the practical considerations of generating perturbations in this manner, and comment on the differences in how we should interpret the results for three different graph-based tasks. The flow of \MethodName is summarized in Figure \ref{fig:method}, and in Algorithms \ref{alg:peanut_nc}-\ref{alg:peanut_gl} in Section \ref{sec:algorithmic desc}.\vspace{10pt}

Before proving the result in Theorem \ref{thm:theorem1}, we first need a key result:
\begin{restatable}[]{lemma}{lemmaone}\label{lemma:main_optimization}
    For a given real-valued $\rmZ\in\sR^{N\times d}$, and budget $\bDelta$, we define the following optimization objective:
    \begin{align}
    \begin{split}
        \max_{\rmB\in\sR^{N\times n}}\quad&\norm{\rmB\rmB\tran \rmZ}_F^2\\
        \rm{s.t.}\qquad&\norm{\rmB}_F^2 \leq \bDelta^2
    \end{split}
    \end{align}
    The solution for the above is $\rmB^* = \bDelta\cdot\rvu_1\rvv\tran$, where $\rvu_1$ is the dominant eigenvector (corresponding to the eigenvalue of largest magnitude) of $\rmZ\rmZ\tran$, and $\rvv$ is any unit-norm vector in $\sR^{n}$.
\end{restatable}

\begin{proof}
    The proof uses the trace-equivalent representation of the Frobenius norm, followed by the eigen-decomposition of $\rmZ\rmZ\tran$. See Appendix \ref{appendix:proof_lemma} for more details.
\end{proof}

\begin{restatable}[]{theorem}{thmone}\label{thm:theorem1}
    Given a trained 2-layer \SGC $\gM_\Theta$, and graph $\gG(\gV,\rmA,\rmX)$, the perturbation $\rmS_v$ that maximizes $\gL(\rmS_v)$ while satisfying $\norm{\rmS_v}_F^2 \leq \bDelta^2$, is given by
    \begin{equation}
        \rmS_v^* = \bDelta\cdot\rvu_1\rvv
    \end{equation}
    where $\rvu_1$ is the dominant eigenvector of $\rmH\rmH\tran$ ($\rmH=\rmX\bTheta$), and $\rvv$ is any vector with unit norm.
\end{restatable}

\begin{proof}
    Using the formulation of $\gL(\rmS_v)$ defined for the SGC in Equation \ref{eqn:sgc_attack}, we can apply Lemma \ref{lemma:main_optimization} to obtain $\rmS_v^*$.
\end{proof}

\paragraph{Working with Limited Information}\label{sec:yyt}
With this simple result, assuming there are no limitations, we can obtain a budget-constrained perturbation $\rmS_v$ that induces the \emph{maximum} change in the output $\rmZ=\gM_\Theta(\rmA,\rmX)$. We can, however, make this work in a black-box manner by using $\rmZ$ instead of $\rmH$. Considering the Cauchy-Schwarz inequality and the formulation in Equation \ref{eqn:sgc_attack},
\begin{align}\label{eqn:blackbox}
    \norm{\rmS_v\rmS_v\tran Z}_F^2 &= \norm{\rmS_v\rmS_v\tran (\rmS^2\rmX\bTheta)}_F^2 \\
    &\leq \norm{\rmS^2}_F^2\norm{\rmS_v\rmS_v\tran}_F^2\norm{\rmX\bTheta}_F^2\\
    &= \norm{\rmS^2}_F^2\cdot\left(\bDelta^4\norm{\rmH}_F^2\right)\label{eqn:blackbx_prefinal}\\
    &= \norm{\rmS^2}_F^2\cdot\gL(\rmS_v)\label{eqn:blackbxend}
\end{align}
Thus, if we assume the attacker only has access to the outputs $\rmZ$, maximizing the LHS using Lemma \ref{lemma:main_optimization} would also push up the $\gL(\rmS_v)$ term in RHS since $\norm{\rmS^2}_F^2$ is constant, albeit at a lesser effect (see Figure \ref{fig:SGC_normdiff}).

\paragraph{Choice of $\rvv$}
For the \SGC, the choice of $\rvv$ has no effect on the outputs since it gets reduced to $1$ when computing $\rmS_{v}\rmS_{v}\tran$. 
This will not happen in more complex MPNN architectures, which may have possible activations or biases inside them, making it a point of further optimization. 
To keep things simple, however, we only adopt two main methods to obtain $\rvv$: \textbf{(1)} sampling $\rvv\in\sR^{\nv}$ from the standard uniform distribution and \textbf{(2)} setting $\rvv=\mathds{1}_\nv$, i.e., a vector of ones. Both are followed by normalization to unit norm.

\paragraph{Norm Differences and its Effectiveness in Classification Tasks}
\MethodName provides a way to augment the graph topology matrix being considered to maximize the norm of the difference of node-level representations of a graph. 
The way \MethodName performs this does not have any control over \emph{where} all the norm difference goes;
for the node classification task, if the effect of perturbation on $\rmZ$ happens to somehow maintain the relative ranking order, the attack will have no effect on the actual predicted classes at test time. 
We will see in Section \ref{sec:results_nc}, however, that \MethodName does display its effectiveness despite this theoretical pitfall. 
We observe the same for the graph classification task; maximizing the norm difference before the readout model $\gM^{out}_{\bTheta_1}$ causes enough perturbation to cause significant misclassifications (Section \ref{sec:results_gc}). 
Compared to the classification datasets, the graph regression task, as expected, is much more sensitive to such an induced norm change.

\paragraph{Applicability to other MPNNs}
Consider a 2-layer \GIN model with the internal MLP being a Linear-ReLU-Linear (LRL) model with no bias. The forward-direction update is as follows:
\begin{align}
    \rmH^{(1)} &= \MLP^{(1)}\left(\left((1+\varepsilon^{(1)})\cdot\rmI + \rmA\right)\rmX\right) \\
    \rmZ &= \MLP^{(2)}\left(\left((1+\varepsilon^{(2)})\cdot\rmI + \rmA\right)\rmH^{(1)}\right) 
\end{align}
It has been observed \citet{xu2018powerfulgin} that when the parameters $\varepsilon^{(1)},\varepsilon^{(2)}$ are set to zero, the resulting \GIN (denoted \GIN-0) consistently outperforms \GIN with nonzero, and even trainable $\varepsilon$, in terms of test accuracy.

Writing this in the form of weights $\rmW_{1-4}$ for the two \MLP\!s, and denoting $\hat{\rmA} = \rmI + \rmA$
\begin{align}
    \rmZ &= \sigma\left( \hat{\rmA} { \rmH^{(1)} }\rmW_3\right)\rmW_4\\
    \rmZ &= \sigma\left( \hat{\rmA} \left({ \sigma(\hat{\rmA}\rmX\rmW_1) \rmW_2 }\right)\rmW_3\right)\rmW_4 
\end{align}

Under a na\"ive assumption of removing nonlinearities, this shows the same form as \SGC. We do not actually use this version of ``\GIN" since it is essentially just an \SGC; however, we want to note that the \emph{pathway} for message passing remains relatively similar. 
While our theoretical results are established for a specific architecture type, we hypothesize that exploiting the message-passing mechanism in the same way remains effective even if the specific model incorporates internally complex message passing mechanisms compared to the \SGC, and validate this empirically in Section \ref{sec:experiments}. We also test the effect of using a perturbation created using $\rvu_1$ as compared to using other heuristic methods instead of the eigenvector-based alignment, such as using the structural properties of the graph as compared to completely random perturbations.

\subsection{Practical Considerations}\label{sec:practical}
\subsubsection{Unnoticeability}
Similar to prior works, we constrain the amount of change $\rmA$ may experience by imposing a norm constraint on the perturbation $\rmA_v$. 
Although including non-zero $\tilde{\rmX}$ may cause the final norm difference to decrease (Equation \ref{eqn:xtilde_effect_on_norm} in Appendix \ref{appendix:sgc_equivalence}), injecting multiple nodes with zero features may be detectable easily, as seen with the Jaccard \citep{wu2019adversarialjaccard} defense method. 
In Section \ref{sec:results_nc_xtilde}, we explore choosing $\tilde{\rmX}$ in different ways and its corresponding effects on performance.
Regarding the injected edges being weighted, the core assumption of this attack is that the MPNN utilizes the graph topology matrix as is, and does not filter it beforehand.
This is common to see in popular library implementations of different architectures, which are then directly used in projects globally.
As MPNNs become increasingly used at scale \citep{geisler2021robustness}, maintaining robust implementations that enforce expected input types without relying on external security methods should become the norm. 
In some applications, however, weighted adjacency matrices may actually be expected as the input, and we see in the effect of \MethodName on the performance on traffic flow prediction, which is one such application, in Section \ref{sec:results_traffic}. 

\subsubsection{Perturbations With Negative Edge Weights}
In the above, we do not constrain positivity on the generated perturbation $\rmS_v$, and using eigenvector $\rvu_1$ means that there is no control over whether the perturbation will be positive or not. 
In principle, if an MPNN simply admits an adjacency matrix as input with no restrictions or checks on whether edge weights are non-negative, this would not be an issue. 
However, considering that these perturbations may be simply ``defended" by using one $\rm ReLU$ call on the adjacency matrix before the propagation steps, we assume as such, and try to minimize the number of perturbation connections that are zeroed out.
This makes the generated perturbations positive with values $\in[0,1]$. 
To this end, we first switch the \emph{direction} of $\rvu_1$ to have the maximum number of positive terms as follows: $\rvu_1$ is assigned a sign based on the number of positive and negative elements in it, using $\rvu_1 = \rvu_1\cdot\sign(\sum\sign(\rvu_1))$. 
With this operation, the eigenvalue ordering remains unaffected as the magnitude is unchanged. 
The method of obtaining $\rvv$ defined above ensures that the number of surviving elements does not experience further change after the $\bDelta\cdot\rvu_1\rvv\tran$ product. 

\subsubsection{A Discretized Version of PEA}
We can discretize the generated perturbation $\rmA_v$ by picking the top-k entries for each virtual node where $k=\max({\rm round}(\bDelta/\nv), 1)$. 
Since we have lesser means to distribute the available budget in this discrete case, we assign at least one edge for each virtual node, making the minimum possible perturbation equal to $2\nv$ ($\nv$ each for $\rmS_v$ and $\rmS_v\tran$).
This is usually only relevant on graph datasets with smaller graphs, where a low $r=0.01$ may correspond to a value less than $1$. 
The effective max budget $\bDelta'$ in this case is essentially $\nv\bDelta$. 
We refer to this version of \MethodName as \MethodName-D.

\subsection{Attack Pipeline}\label{sec:algorithmic desc}
In Algorithms \ref{alg:peanut_nc} and \ref{alg:peanut_gl}, we denote $K$ as the number of test graphs. The function $\textsc{DomEigvec}$ returns the dominant eigenvector (corresponding to the eigenvalue with maximum magnitude) of the input. In our experiments, since $\rmZ\rmZ\tran$ is a real symmetric matrix, we use the optimized eigenvector solver in PyTorch (\verb|torch.linalg.eigh|), however, the dominant eigenvector can also be found using power iterations or any alternative method. For other variants of \MethodName and the heuristic baseline methods, $\rmA_v$ is modified before the \emph{Insertion} step in these algorithms.

\begin{algorithm}[h]
\caption{\textsc{\MethodName} for NC with a single large graph.}
\label{alg:peanut_nc}
\begin{algorithmic}[1]
    \Require Graph $\gG(\rmV,\rmA,\rmX)$ with $N$ nodes, Budget ratio $r$
    \Ensure Perturbed Graph $\gG_p(\rmV_p,\rmA_p,\rmX_p)$

    \State $\rmZ \gets \gM_{\bTheta}\left(\rmA,\rmX\right)$ 
    \State $\rvu_1 \gets {\textsc{DomEigvec}}(\rmZ\rmZ\tran)$
    \State /* Aligning $\rvu_1$ along the direction that maintains maximum positive items */
    \State $\rvu_1 \gets \rvu_1\cdot\sign(\sum\sign(\rvu_1))$
    \State $\nv \gets {\rm floor}(rN), \quad \bDelta \gets {\rm floor}\left(\sqrt{r\cdot N \cdot \deg(\gG)}\right)$
    \State $\rvv \gets \frac{1}{\sqrt{\nv}} {\textsc{Ones}}(\nv)$ 
    \State $\rmA_v \gets {\rm ReLU}(\bDelta\cdot\rvu_1\rvv\tran)$
    \State $\tilde{\rmX} \gets \boldsymbol{0}^{\nv\times D}$
    \State $\rmA_p \gets \begin{pmatrix} \rmA & \rmA_v \\ \rmA_v\tran & \boldsymbol{0} \end{pmatrix}$, $\rmX_p \gets \begin{pmatrix} \rmX \\ \tilde{\rmX}\end{pmatrix}$ \Comment{Insertion}
    \State $\gG_p \gets (\gV \cup \{\nv\},\rmA_p, \rmX_p)$
\end{algorithmic}
\end{algorithm}

\begin{algorithm}[h]
\caption{\textsc{\MethodName} for Graph-level tasks.}
\label{alg:peanut_gl}
\begin{algorithmic}[1]
    \Require Test Set of Graphs $\gG^{(1)},\gG^{(2)},...,\gG^{(K)}$, Trained MPNN $\gM_\bTheta$. Budget ratio $r$, Number of virtual nodes $\nv$
    \Ensure Perturbed graphs $\gG^{(1)}_p,\gG^{(2)}_p,...,\gG^{(K)}_p$

    \For{$i\in[1,...,K]$} \Comment{Graphs can also be batched}
        \State $\rmZ \gets \gM^{(agg)}_{\bTheta_2}\left(\rmA^{(i)},\rmX^{(i)}\right)$ 
        \State $\rvu_1 \gets {\textsc{DomEigvec}}(\rmZ\rmZ\tran)$
        \State $\rvu_1 \gets \rvu_1\cdot\sign(\sum\sign(\rvu_1))$ 
        \Comment{Aligning $\rvu_1$ along the direction which maintains maximum positive items}
        \State $\rvv \gets \frac{1}{\sqrt{\nv}} {\textsc{Ones}}(\nv)$ 
        \State $\bDelta\gets {\rm floor}(\sqrt{r\cdot |\gE^{(i)}|})$
        \State $\rmA_v \gets {\rm ReLU}(\bDelta\cdot\rvu_1\rvv\tran)$
        \State $\tilde{\rmX} \gets \boldsymbol{0}^{\nv\times D}$
        \State $\rmA_p \gets \begin{pmatrix} \rmA & \rmA_v \\ \rmA_v\tran & \boldsymbol{0} \end{pmatrix}, \quad\rmX_p \gets \begin{pmatrix} \rmX \\ \tilde{\rmX}\end{pmatrix}$ \Comment{Insertion}
        \State $\gG^{(i)}_p \gets (\gV \cup \{\nv\},\rmA_p, \rmX_p)$
    \EndFor
\end{algorithmic}
\end{algorithm}

\section{Experiments}\label{sec:experiments}
In this section, we empirically evaluate \MethodName on three graph tasks---Node Classification (NC), Graph Classification (GC), and Graph Regression (GR)---and explore its effects on traffic forecasting as a node-level regression problem.

\input{tables/bigtable_sgc_vs_baselines}

\input{tables/bigtable_sgc_vs_heuristics}

\subsection{Experimental Setup}
\subsubsection{Datasets}
We evaluate the efficacy of \MethodName on NC using the three benchmark citation networks---Cora, Citeseer, and Pubmed \citep{citationdatasets}, and perform an in-depth comparison with baselines.
Then, we evaluate \MethodName on five benchmark regression datasets---ESOL, FreeSolv, and Lipophilicity \citep{moleculenet}, ZINC \citep{zincdataset}, and AQSOL \citep{aqsoldataset}, and on GC using four benchmark graph classification datasets from TUDataset \citep{tudataset}---MUTAG, PROTEINS, ENZYMES, and IMDB-BINARY---and two datasets from MoleculeNet \citep{moleculenet}---BBBP and BACE, for a total of six datasets spanning molecular property prediction, social networks, biophysics, and physiology. 
We summarize some key statistics of the NC datasets in Table \ref{tab:stats_nc}, and the graph-level task datasets in Table \ref{tab:stats_gcgr}.

\begin{table}[H]
\centering
\caption{Dataset summary for the NC task.} %
\begin{tabular}{lccc}
\toprule
Statistic $\backslash$ Dataset & Cora & Citeseer & Pubmed \\
\midrule
Number of Classes & 7 & 6 & 3 \\
Number of Nodes & 2708 & 3327 & 19717 \\
Number of Edges & 10556 & 9104 & 88648 \\
Avg. Degree & 4.90 & 3.74 & 5.50 \\
$\nv$ for $r=0.01$ & 27 & 33 & 197 \\
$\bDelta$ for $r=0.01$ & 132 & 124 & 1083 \\
\bottomrule
\end{tabular}
\label{tab:stats_nc}
\end{table}

\subsubsection{Metrics}
For the regression datasets, we report the test RMSE and MAE, and for the classification datasets, we compute the accuracy and macro-averaged F1 score (denoted by F1) on the test set, both as percents. 
Unless otherwise mentioned, all reported numbers have been averaged over $10$ runs, with the deviation error-bars being left out from the plots to avoid excess clutter.

\subsubsection{Attack Budgets}\label{sec:results_budgets}
For NC, we follow the standard budget allowances from literature \citet{sun2020adversarial}, keeping the number of virtual nodes $\nv=r\cdot|V|$ and $\bDelta=\sqrt{r\cdot|V|\cdot\deg(\gG)}$, where $\deg(\gG)$ is the average degree of the graph. 
For GC and GR, we consider the budget $\bDelta$ available to the attacker as a fraction $r$ of the number of edges of each individual graph, i.e., $\bDelta=\sqrt{r\cdot|\gE^{(i)}|}$ for graph $\gG^{(i)}$, and we vary $\nv\in[0,1,2,5,10]$.

\subsubsection{Baselines}\label{sec:reslts_baselines}
Since node injection attacks remain underexplored on the GC and GR tasks, we compare the performance of \MethodName with baselines on the well-studied and benchmarked NC task. We use TDGIA \citep{zou2021tdgia}, ATDGIA, a variant of TDGIA proposed by \citep{chen2022understanding}, AGIA \citep{chen2022understanding}, and the more recent G2A2C \citep{ju2023let} as four key baselines. 
TDGIA involves topological defective edge selection followed by feature learning using smooth adversarial optimization. AGIA requires gradient information (via surrogate) to learn the features and structures of the injected nodes. 
For a fair comparison, we impose our attack budget constraints defined in Section \ref{sec:results_budgets} to decide the number of injected nodes and the corresponding $\bDelta$.

In addition to these, we also include four heuristic baselines for comparing the effectiveness of \MethodName on the other tasks---\textbf{(1)} RAND, where the perturbations are sampled randomly from the Normal distribution; \textbf{(2)} BET, which uses the betweenness centrality of nodes; \textbf{(3)} DEG, which defines the perturbation using the average degrees of nodes; and \textbf{(4)} CENT, which uses the eigenvector centrality of the adjacency matrix $\rmA$. For the methods other than RAND, a single vector is computed with the corresponding property for each node (or the dominant eigenvector in the case of CENT) followed by normalization to unit norm. This vector is repeated $\nv$ times to create the complete perturbation matrix $\rmA_v$ (or $\rmS_v$), which is then scaled using $\bDelta$.

\subsubsection{Architectures}\label{sec:architectures}
For NC, we use the \SGC \citep{wu2019simplifying}, \GCN \citep{kipf2016semi}, \GIN \citep{xu2018powerfulgin}, and \SAGE \citep{hamilton2017sage} architectures for evaluating the effects of \MethodName-style perturbations across models with different propagation styles. 
The \GIN and \SAGE architectures usually do not admit the adjacency matrix as-is since they only deal with unattributed aggregation over neighbors. 
However, common implementations often use products of the form $\rmA\rmX$ as a parallelizable and efficient aggregation step to sum over neighbors, which is an exploitable vulnerability for \MethodName. 
If we assume direct access to $\rmS$ instead of only access to $\rmA$, it describes the true perturbation optimum from Section \ref{sec:method}. For the two architectures that utilize normalized adjacency matrices---\SGC and \GCN---we consider this scenario separately, denoting the layer as S-\SGC and S-\GCN respectively in Figures and Tables.
The architecture of choice for graph-level tasks will be a 2-layer \GIN. However, for a few simple graph-regression datasets, we also train a \GCN to confirm our performance trends. For all tasks, each architecture described consists of 2 graph convolution layers of the specified architecture.

\subsection{Results on Node Classification}\label{sec:results_nc}

\subsubsection{Comparison with Baselines}
In Table \ref{tab:SGC_baselines}., we compare the performance of \MethodName with the four baselines. 
Despite not being designed explicitly for node classification, and with the pitfalls discussed in Section \ref{sec:method} when dealing with classification-based tasks, \MethodName performs well across the board, and significantly outperforms the baselines on slightly higher budgets ($\sim 5\%$ injected nodes). 
In Table \ref{tab:SGC_heuristics}, we compare with the heuristic baselines, and in Figure \ref{fig:sgc_heuristics_nof1}, we visualize the accuracy along with the effect of perturbation $\norm{\rmZ_p - \rmZ}_F$, and note similar trends. Notably, even random injection is severely damaging under this threat model, with greater performance degradation than the non-heuristic baselines.
In Appendix \ref{appendix:heuristics}, we compare the performance of all heuristic baselines across all model architectures.

\subsubsection{Note on Adjacency Normalization}\label{sec:results_ssgc_sgcn}
If the \GCN or \SGC here do not normalize the adjacency matrix inside and expect the normalized adjacency matrix $\rmS$ as input, and we can perturb that directly, we observe an astonishing accuracy (F1) drop of over $59(80)\%$ on the Cora dataset, with Citeseer experiencing a $32(42)\%$ drop, and Pubmed experiencing a $43(63)\%$ drop, even with $r$ set as low as $0.001$. Perturbations added to the normalized adjacency matrix directly seem to cause it to go so far out of distribution that they collapse to a specific state.

\begin{figure}[h]
    \centering\includegraphics[width=\linewidth]{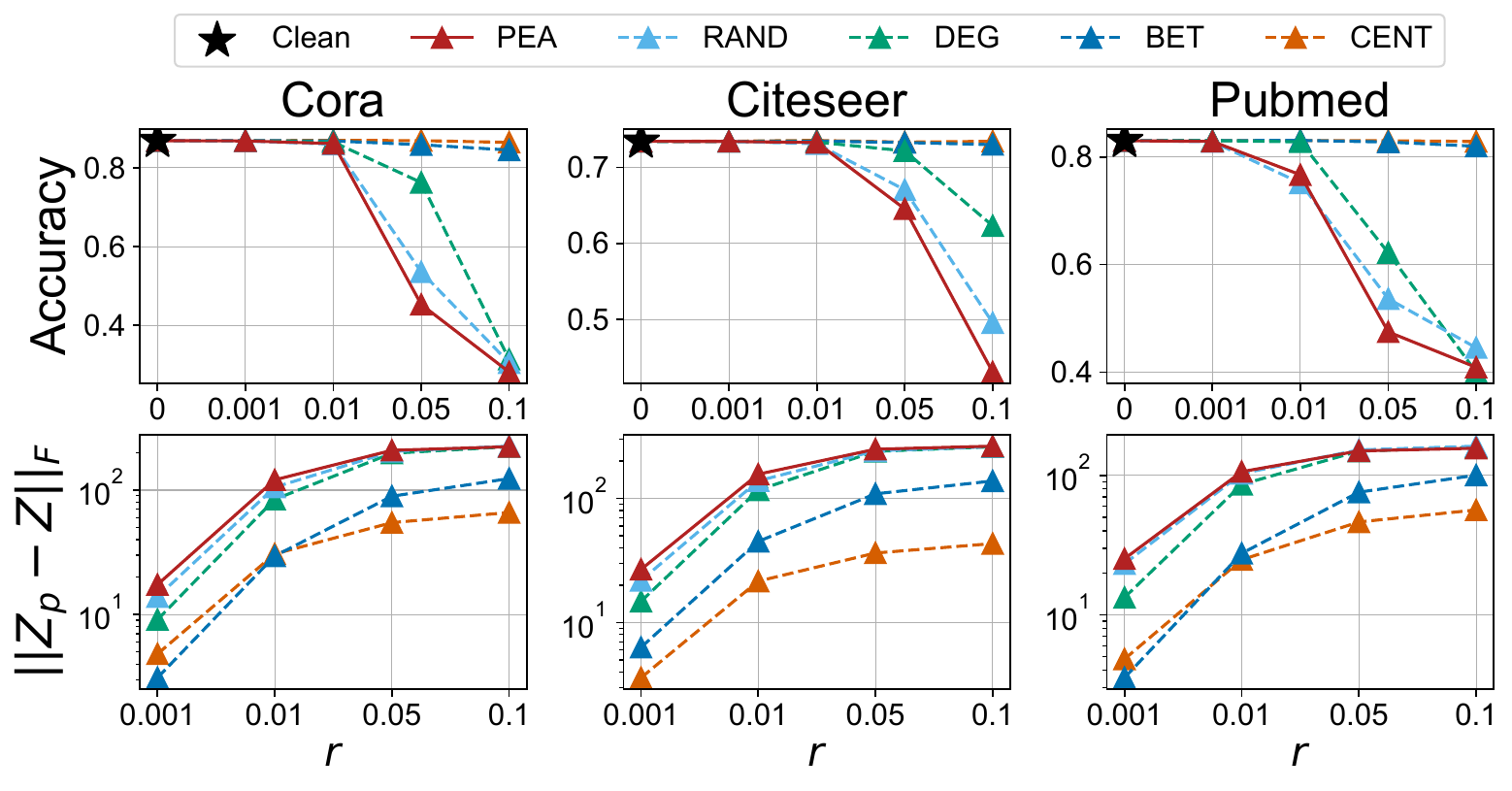}%
    \caption{Comparing \MethodName to the heuristic baselines across the three NC datasets, on the \SGC.}
    \label{fig:sgc_heuristics_nof1} \vspace{-1em}
\end{figure}

\subsubsection{\MethodName across MPNN Architectures}\label{sec:sage_discussion}
In Table \ref{tab:nc_architectures} (left), we apply \MethodName to the four different MPNN architectures over two values of $r$. 
Across all the NC experiments, \SAGE seems extremely resilient to \MethodName. Comparing it with the \GIN, the core difference beyond the degree of parameterization is the difference in aggregation styles---by default, \SAGE uses the mean aggregation and \GIN uses the sum aggregation. 
In Table \ref{tab:sage_ablation}, we explore the effect of modifying the internal mechanism of the \SAGE and \GIN models, switching their aggregation methods, and observing the effects. 
Remarkably, simply switching \GIN's aggregation to what \SAGE uses also carries the resilience with it. 
Internally, \SAGE computes the hidden state as $\rmH = \rmA\rmX$ followed by a row-wise division by ${\sum_{j=1}^N A_{ij}}$. In the binary adjacency case, would correspond to the mean-aggregated features of each node and its neighbors. This is then followed by two \MLP\!s: one acting on $\rmH$ and one acting on $\rmX$, the outputs of which are added together to give the final representation, completing the action of one layer. 
The equivalent operation in the \GIN is simply $\rmH = \rmA\rmX$, which corresponds to sum-aggregation over neighbors. Then, a single \MLP acts on $(1+\varepsilon)\rmX + \rmH$. 
With this in mind, the $\sum_{j=1}^N A_{ij}$ term in the denominator seems to weigh down the perturbed edges significantly more than the original binary edges since the \MethodName perturbations are continuous with the norm $\bDelta$ spread evenly across all nodes, thereby reducing its effect as a whole. However, as we will see in Section \ref{sec:results_nc_xtilde} below, the resilience of \SAGE is also broken by using non-zero injected features.

\begin{table}[h]
\caption{Ablation on \SAGE and \GIN, with a fixed $r=0.05$. Results on more values of $r$ are reported in Appendix \ref{appendix:sage_ablation}. The original pairing is \SAGE and \GIN with Mean and Sum aggregation respectively.}\label{tab:sage_ablation}%
\resizebox{\linewidth}{!}{
\begin{tabular}{llrrrr}
\toprule
\multirow{2}{*}{\vspace{-3pt}Dataset} & \multirow{2}{*}{\vspace{-3pt}Layer} & \multicolumn{2}{c}{Mean Aggregation} & \multicolumn{2}{c}{Sum Aggregation} \\\cmidrule(lr){3-4}\cmidrule(lr){5-6}
 &  & \multicolumn{1}{c}{Accuracy$\downarrow$} & \multicolumn{1}{c}{F1$\downarrow$} & \multicolumn{1}{c}{Accuracy$\downarrow$} & \multicolumn{1}{c}{F1$\downarrow$} \\\midrule
\multirow{2}{*}{Cora} & SAGE & 3.09±0.65 & 3.04±0.64 & 49.26±7.81 & 68.47±8.14 \\
 & GIN & 4.70±0.89 & 5.42±1.12 & 51.74±3.79 & 72.28±3.69 \\\midrule
\multirow{2}{*}{Citeseer} & SAGE & 0.65±0.35 & 0.66±0.57 & 49.26±7.81 & 68.47±8.14 \\
 & GIN & 1.79±1.07 & 1.78±1.28 & 34.56±14.10 & 44.31±15.50 \\\midrule
\multirow{2}{*}{Pubmed} & SAGE & 0.79±0.22 & 0.82±0.26 & 44.94±1.16 & 64.75±1.91 \\
 & GIN & 11.28±4.52 & 16.65±9.45 & 40.04±9.28 & 58.70±11.33 \\\bottomrule
\end{tabular}
}\vspace{0.5em}
\end{table}

\input{tables/nc_arch_xtilde_subfloat}

\subsubsection{Effects of Choosing $\tilde{\rmX}$ Differently}\label{sec:results_nc_xtilde}
As discussed in Section \ref{sec:practical}, injecting a large number of nodes, all with zero features, may be an easily detectable attack. Here, we explore five alternative ways of choosing $\tilde{\rmX}$---avg, rand, randn, randsamp, and randd---and their effect on attack performance. 
In the \emph{avg} mode, we choose $\tilde{\rmX}$ as the mean value of all features in the original graph; in the \emph{rand} and \emph{randn} modes, we sample from the standard uniform and normal distributions, followed by normalization to the average feature norm; and in the \emph{randsamp} mode, we randomly sample $\nv$ features from $\rmX$. The \emph{randd} mode is for datasets that consist of discrete features, like Cora and Citeseer. Features in this mode are created by creating a one-hot vector by randomly choosing $f_{avg}={\textsc{Mean}}(\sum_j X_{ij})$ feature indices and setting those to $1$. 
In Table \ref{tab:heuristics_vs_xtilde_mini}, we implement all six ways of choosing $\tilde{\rmX}$ for the Cora and Citeseer datasets, with a fixed budget of $r=0.05$. In most cases, including non-zero features on the injected virtual nodes causes further performance degradation; however, as mentioned earlier, this may not be guaranteed if $\rmX$ itself does not follow desirable properties (positivity, whether features are discrete or continuous, etc.), which may cause the norm in Equation \ref{eqn:xtilde_effect_on_norm} to decrease. From these results, the effect of (scaled) random uniform features on \SAGE is the most notable, demonstrating that with perturbing a graph, exploring the additional degree of freedom and optimizing for $\tilde{\rmX}$ instead of choosing it using heuristics may lead to even better results, across more complex architectures.

\subsubsection{Note on Output Norm Differences}
In Figure \ref{fig:SGC_normdiff}, we compare the perturbation effect $\gL(\rmS_v)$ for the NC datasets, supporting the result of Theorem \ref{thm:theorem1} and demonstrating that the black-box approximation in Equations \ref{eqn:blackbox}-\ref{eqn:blackbxend} does follow the expected trends of when accounting for $\norm{\rmS^2}_F$. 
\begin{figure}[H]
    \centering
    \includegraphics[width=1\linewidth]{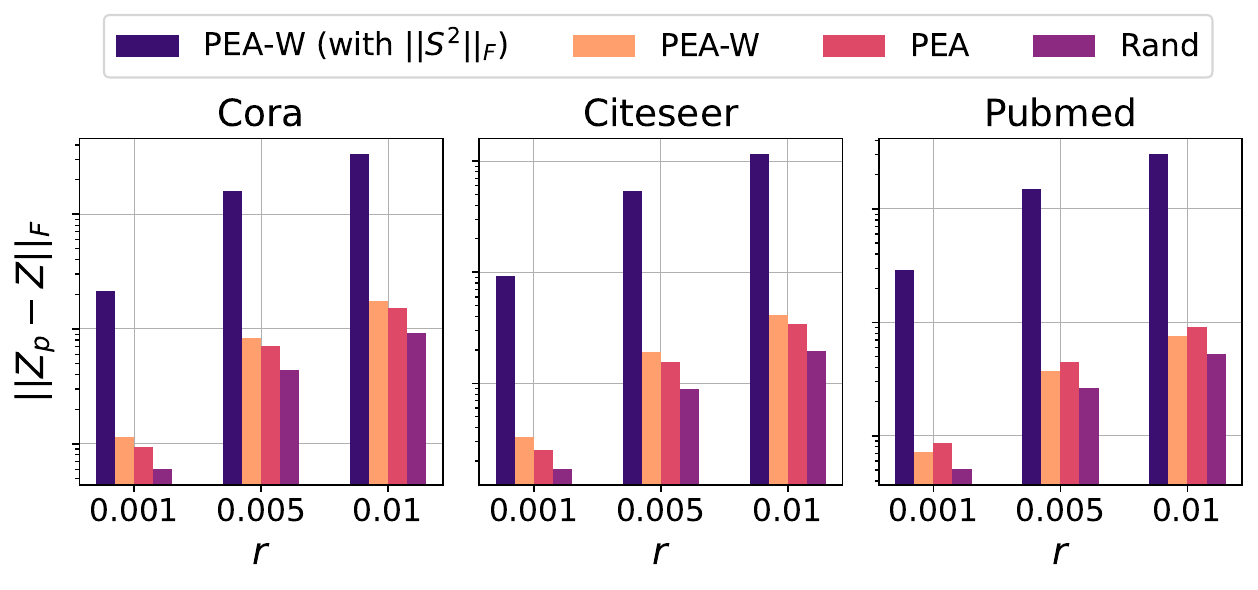}\vspace{-1em}
    \caption{Norm difference ($\gL(\rmS_v)$) between \MethodName-W, \MethodName, and a randomly chosen $\rmS_v$, illustrating how the white-box version compares with the black-box approximation and randomly chosen perturbations, using \SGC on the NC datasets.}
    \label{fig:SGC_normdiff}%
\end{figure}
In the Figure, we clearly see that the approximation and, consequently, maximizing the LHS in Equation \ref{eqn:blackbox} does indeed lead to a comparable, and in the case of Pubmed, better, $\gL(\rmS_v)$ compared to the \MethodName-W. In Figure \ref{fig:sgc_heuristics_nof1}, we also plot the norm difference $\norm{\rmZ_p - \rmZ}_F$ induced by \MethodName and the heuristic baselines, and note that \MethodName is consistently higher, despite all methods working on the same threat model (injecting weighted perturbations). In Appendix \ref{appendix:heuristics}, we plot these trends for all model types across the three datasets, and note that the models where we can perturb $\rmS$ directly suffer significantly larger $\gL(\rmS_v)$ compared to the others. As discussed in Section \ref{sec:results_ssgc_sgcn}, since this scenario directly reflects the one for which we have proven the theoretical optimum, this is expected.

\begin{table*}[ht]
\centering
\caption{Dataset summary for the GR and GC tasks.}%
\resizebox{\linewidth}{!}{
\begin{tabular}{lccccccccccc}
\toprule
Statistic $\backslash$ Dataset & FreeSolv & ESOL & Lipophilicity & ZINC & AQSOL & MUTAG & PROTEINS & ENZYMES & IMDB-BINARY & BBBP & BACE \\
\midrule
Number of Graphs & 641 & 1126 & 4200 & 12000 & 9833 & 188 & 1113 & 600 & 1000 & 2039 & 1513 \\
Number of Classes & - & - & - & - & - & 2 & 2 & 6 & 2 & 2 & 2 \\
Avg. Number of Nodes & 8.73 & 13.29 & 27.04 & 23.16 & 17.58 & 17.93 & 39.06 & 32.63 & 19.77 & 24.06 & 34.09 \\
Avg. Number of Edges & 25.51 & 40.63 & 86.04 & 72.99 & 53.38 & 57.52 & 184.69 & 156.91 & 212.84 & 75.97 & 107.81 \\
Avg. Degree & 2.83 & 2.98 & 3.18 & 3.14 & 2.98 & 3.19 & 4.73 & 4.86 & 9.89 & 3.13 & 3.17 \\
Avg. $\bDelta$ for $r=0.05$ & 1.28 & 2.03 & 4.30 & 3.65 & 2.67 & 2.88 & 9.23 & 7.85 & 10.64 & 3.80 & 5.39 \\
\bottomrule
\end{tabular}}\vspace{1em}
\label{tab:stats_gcgr}
\end{table*}

\subsubsection{Performance Against Filter-based Defenses}
In Table \ref{tab:filter_defense_cora_no_f1}, we test the efficacy of the Jaccard \citep{wu2019adversarialjaccard} and SVD \citep{entezari2020allsvd} graph defense methods on the Cora dataset, using the \SGC. Since Jaccard works on computing feature similarities, we choose $\tilde{\rmX}$ using the \emph{randsamp} mode, randomly sampling features from the graph. The numbers in the bracket for Jaccard refer to the Jaccard similarity threshold, which determines which edges are filtered out, and for SVD, they refer to the $k$-rank approximation for the adjacency matrix. 
Even with a different $\tilde{\rmX}$, \MethodName does not work much on the feature side of things, and as such, Jaccard is able to defend against it, although suffering a decline in performance on the original dataset. SVD is unable to filter out the perturbation despite working on the perturbed adjacency matrix directly, and suffers a much more significant performance drop if $k$ is set very low.

\subsection{Results on Graph-Level Tasks}\label{sec:results_gr}
Table \ref{tab:stats_gcgr} summarizes some key statistics for the five GR and six GC datasets.

\subsubsection{Graph Regression}

\begin{figure*}[htb]
    \centering\includegraphics[width=0.95\linewidth]{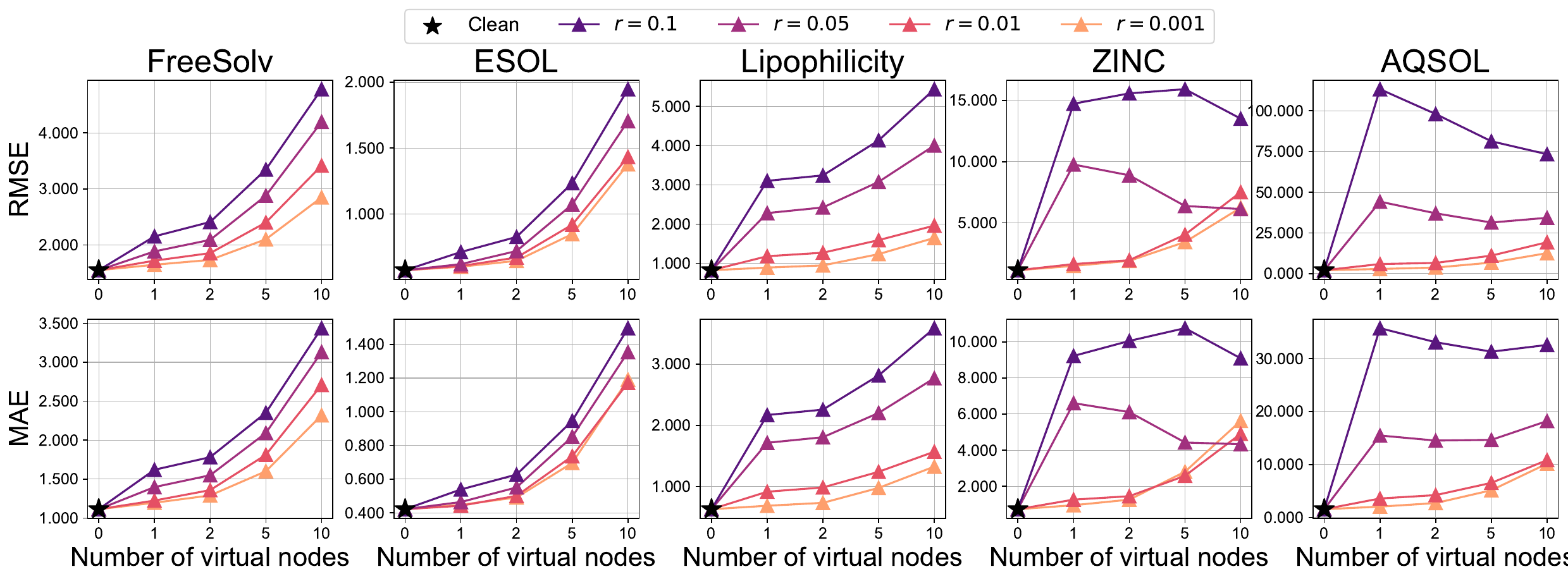}
    \caption{\GIN Regression performance vs number of virtual nodes on the five regression datasets.}
    \label{fig:GINA_vs_budget}
\end{figure*}

For the regression task, we implement two main model classes: \textbf{(1)} A simple 2-layer \GCN, and \textbf{(2)} A 2-layer \GIN. As mentioned in Section \ref{sec:architectures}, although our main architecture is \GIN, we train a \GCN on the relatively easier datasets to confirm our hypotheses on performance trends. We visualize the impact of \MethodName in Figures \ref{fig:GINA_vs_budget} and \ref{fig:GCNA_vs} as the number of injected nodes $\nv$ increases, plotting both the RMSE and MAE of the output predictions. 

In Figure \ref{fig:GCNA_vs}, we observe the expected degradation in regression performance when the number of injected nodes and perturbation budgets are increased for the \GCN. Despite having much better performance on the core task, in Figure \ref{fig:GINA_vs_budget}, we see the same trends for the \GIN, lending some empirical support to our hypothesis that exploiting the message-passing mechanism in the way described works even for architectures with different internal architectures (Section \ref{sec:practical}). 

\begin{figure}[H]
    \centering
    \includegraphics[width=\linewidth]{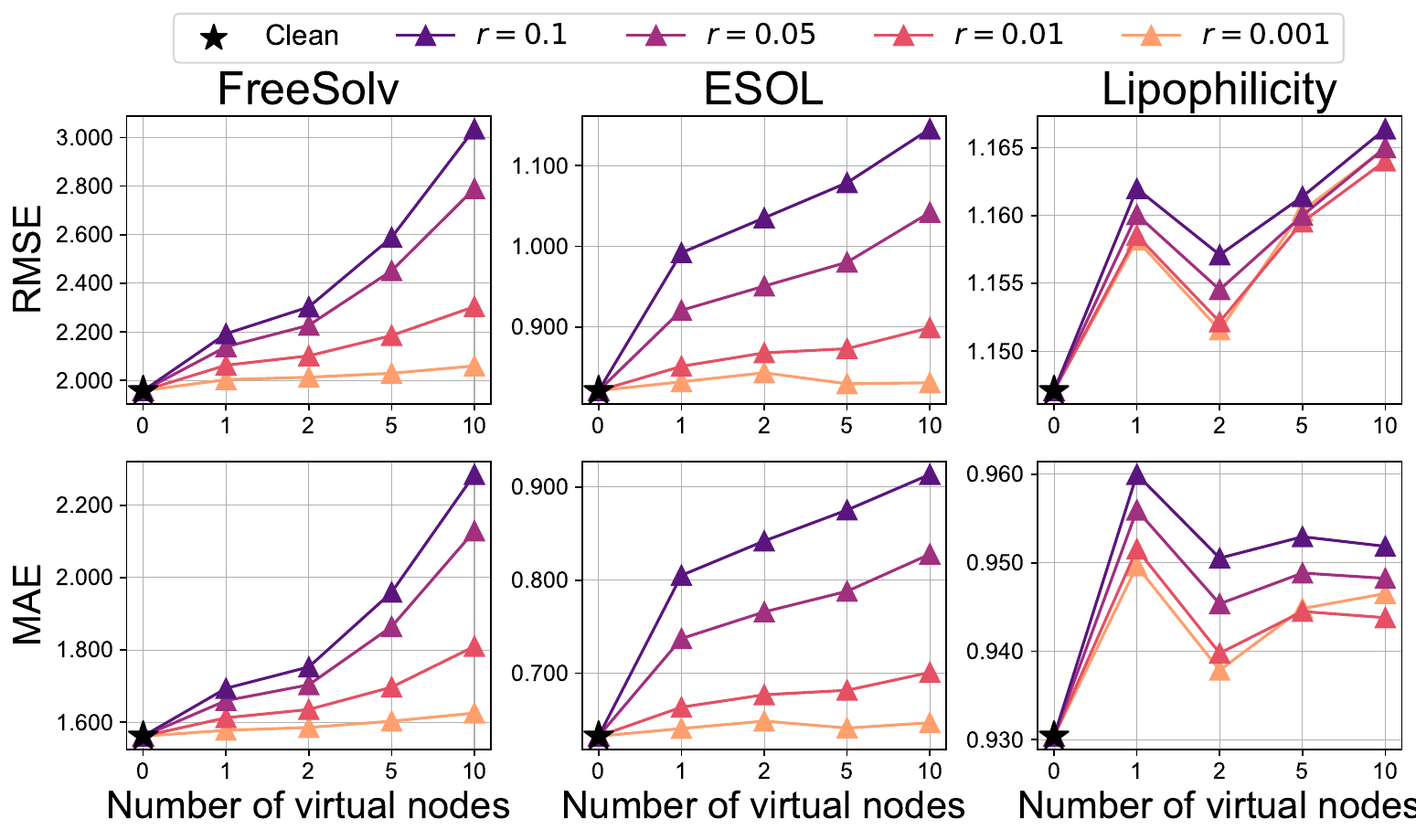}
    \caption{\GCN Regression performance vs number of virtual nodes on the MoleculeNet datasets.}
    \label{fig:GCNA_vs}
\end{figure}

\begin{figure*}[t]
    \subfloat{
        \includegraphics[width=0.95\linewidth]{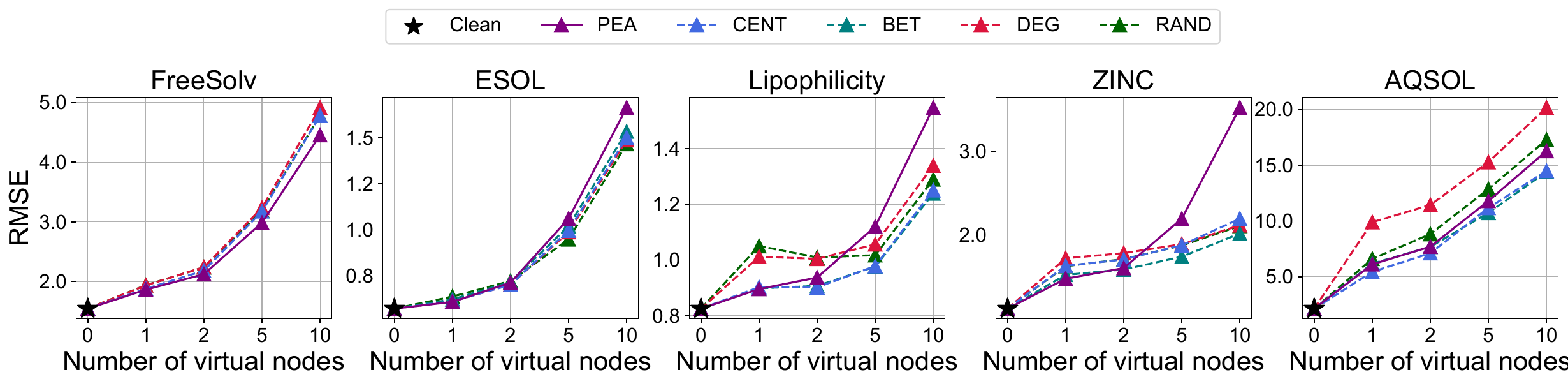}
    }
    \vfil
    \subfloat{
        \includegraphics[width=0.95\linewidth]{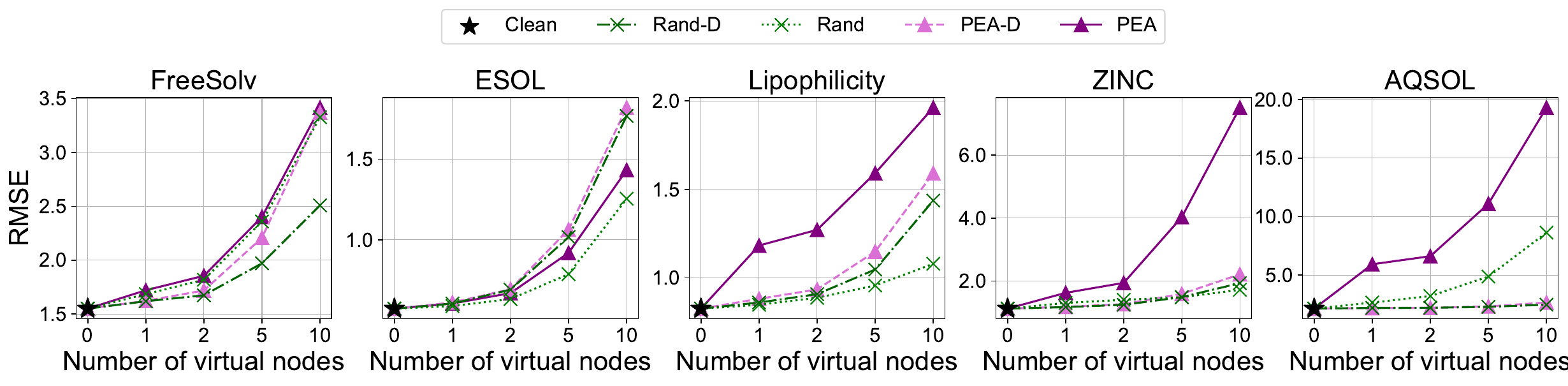}
    }
    \caption{\GIN regression performance across the five datasets after perturbation with \MethodName and the four heuristic baselines (Top); and with \MethodName, \MethodName-D, \rand and \rand-D (Bottom), for the same budget (normalized for the discrete variants).}\label{fig:gr_heuristics}
\end{figure*}

\subsubsection*{Performance vs heuristics}
In Table \ref{tab:gr_heuristics}, we compare \MethodName with the four heuristic baselines: RAND, DEG, BET, and CENT, with the number of injected virtual nodes fixed to $5$. \MethodName remains competitive with the baselines, with the best attack performance on three datasets.

In Figure \ref{fig:gr_heuristics} (Top), we perform this comparison across different values of $\nv$. We also compare with the discrete perturbation variants; in Figure \ref{fig:gr_heuristics} (Bottom), we compare our base method with three alternatives: the discrete variant \MethodName-D discussed in \ref{sec:practical}, \rand, and \rand-D, which is discretized in the same way as \MethodName-D.

\begin{table}[h]
\centering
\caption{Performance of \MethodName and the heuristic baselines on the GR datasets with $\nv=5$.}%
\begin{tabular}{lcccccc}
\toprule
Dataset & Clean & PEA & RAND & DEG & BET & CENT \\
\midrule
FreeSolv & 1.113 & 2.145 & 2.277 & 2.345 & 1.724 & \underline{2.356} \\
ESOL & 0.422 & \underline{0.858} & 0.767 & 0.801 & 0.848 & 0.806 \\
Lipophilicity & 0.629 & \underline{0.861} & 0.786 & 0.830 & 0.758 & 0.762 \\
ZINC & 0.729 & \underline{1.751} & 1.610 & 1.614 & 1.455 & 1.594 \\
AQSOL & 1.545 & 8.243 & 9.681 & \underline{11.505} & 8.031 & 9.243 \\
\bottomrule
\end{tabular}\vspace{-1.5em}
\label{tab:gr_heuristics}
\end{table}

\subsubsection{Graph Classification}\label{sec:results_gc}
For graph classification, we use the \GIN due to its significant performance over other architectures across all datasets. In Table \ref{tab:clean_acc_gc}, we report the accuracy and macro-averaged F1 scores on the clean datasets. 

\begin{table}[H]
\centering
\caption{Mean performance of the \GIN on GC. Statistics are computed over 10-Fold cross validation splits.}
\begin{tabular}{lcc}
\toprule
Datasets & Accuracy & Macro F1 \\
\midrule
MUTAG & $84.18${\scriptsize$\pm8.12$} & $81.10${\scriptsize$\pm10.75$} \\
PROTEINS & $75.11${\scriptsize$\pm2.43$} & $73.67${\scriptsize$\pm2.32$}  \\
ENZYMES & $44.00${\scriptsize$\pm8.10$} & $42.94${\scriptsize$\pm6.65$} \\
IMDB-BINARY & $71.30${\scriptsize$\pm3.80$} & $71.05${\scriptsize$\pm3.84$} \\
BBBP & $87.49${\scriptsize$\pm1.49$} & $81.24${\scriptsize$\pm3.43$} \\
BACE & $78.26${\scriptsize$\pm3.26$} & $77.93${\scriptsize$\pm3.39$}\\
\bottomrule
\end{tabular}
\label{tab:clean_acc_gc}
\end{table}

In Figure \ref{fig:GINA_gc}, similar to the GR task, we observe a drop in accuracy as we increase the number of injected nodes. In this task, we see that increasing the budget does not always lead to a further drop in performance (notably, in the IMDB-BINARY and MUTAG datasets). 
This occurs due to how \MethodName works on norms and not explicitly targeting a certain class, leading to cases where increasing the norm difference of a certain model output may very well place \emph{more} weight on the output logits (or probabilities) corresponding to the correct class, thus not affecting the model's prediction. 
The GC task is the farthest from our setting, not only including a feature aggregation step to generate a graph-level representation, but also performing classification on top, which suffers from the pitfall discussed, and thus we do not explore heuristic baselines for it.
Even in this case, however, we \emph{do} observe performance drops, albeit with some expected trend violations. In Appendix \ref{appendix:gc_additional} (Table \ref{fig:GINA_gc_f1}), we also report the effect of \MethodName on the F1 score as an additional metric, noting the same trends in performance degradation.

\begin{figure*}[htb]
    \centering
    \includegraphics[width=\linewidth]{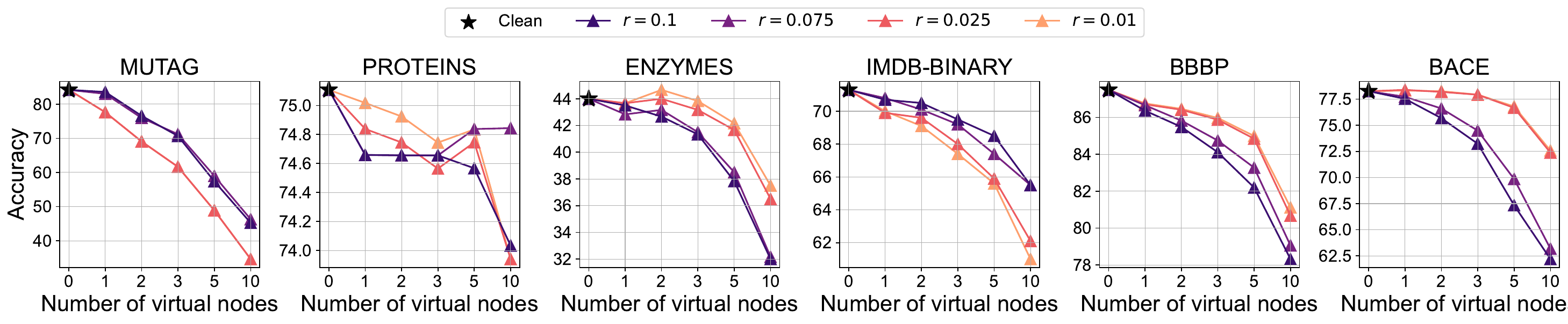}
    \caption{\GIN Classification accuracy vs number of virtual nodes on the six classification datasets. All reported numbers have been averaged over the 10-Fold CV.}\vspace{-1em}
    \label{fig:GINA_gc}
\end{figure*}

\subsection{Traffic Flow Prediction}\label{sec:results_traffic}
In this section, we apply \MethodName and the heuristic baselines to a task that expects a weighted adjacency matrix as the input---Traffic flow prediction---where a adjacency matrix is often constructed using inter-node distances using the thresholded Gaussian kernel \citep{shuman2013emerging}. 

Due to the difficulty of the task and the involvement of a temporal dimension, the relatively simple MPNNs we have considered do not perform well. Thus, we implement the STGCN \citep{yu2017spatio} on three datasets---PEMS-BAY (325 nodes) from \citep{li2017dcrnn}, and PEMSD7(M) (228 nodes) and PEMSD7(L) (1026 nodes) datasets from \citep{yu2017spatio}. 
For the metrics, we follow \citep{yu2017spatio} and report the MAE, MAPE (in \%), and RMSE, with a fixed forecast horizon of $1$\,hr ($12$ values). In Table \ref{tab:traffic}, although better than random perturbations, \MethodName does not always perform better than perturbations chosen using heuristics, with DEG working exceptionally well over all three datasets.

\begin{table}[ht]
\centering
\caption{Performance on three metrics on the three traffic flow prediction datasets. The numbers below the metric name denote the clean performance, and the numbers for each method denote the respective \textit{increases} (higher is better).}
\resizebox{\linewidth}{!}{
\begin{tabular}{cccccccc}
\toprule
Dataset & Metric & Budget  & PEA & RAND & DEG & BET & CENT \\
\midrule
\multirow[c]{9}{*}{PEMS-BAY} & \multirow[c]{3}{*}{\shortstack{MAE\\(1.790)}} & 0.050 & 0.282 & 0.174 & \underline{0.334} & 0.073 & 0.283 \\
 &  & 0.075 & 0.548 & 0.287 & \underline{0.688} & 0.121 & 0.635 \\
 &  & 0.100 & 0.833 & 0.456 & \underline{1.122} & 0.187 & 1.050 \\
\cmidrule{2-8}
 & \multirow[c]{3}{*}{\shortstack{MAPE\\(4.136)}} & 0.050 & 0.504 & 0.284 & \underline{0.575} & 0.237 & 0.504 \\
 &  & 0.075 & 1.042 & 0.496 & \underline{1.165} & 0.363 & 1.138 \\
 &  & 0.100 & 1.609 & 0.811 & \underline{1.936} & 0.517 & 1.888 \\
\cmidrule{2-8}
 & \multirow[c]{3}{*}{\shortstack{RMSE\\(4.144)}} & 0.050 & 0.046 & 0.032 & \underline{0.152} & -0.131 & 0.020 \\
 &  & 0.075 & 0.252 & 0.132 & \underline{0.590} & -0.070 & 0.501 \\
 &  & 0.100 & 0.510 & 0.288 & \underline{1.191} & 0.006 & 1.128 \\
\midrule
\multirow[c]{9}{*}{PEMSD7(M)} & \multirow[c]{3}{*}{\shortstack{MAE\\(3.237)}} & 0.050 & 1.246 & 0.981 & \underline{1.571} & 0.737 & 0.296 \\
 &  & 0.075 & \underline{3.330} & 1.926 & 3.130 & 1.308 & 0.445 \\
 &  & 0.100 & \underline{5.161} & 3.480 & 4.924 & 1.969 & 0.612 \\
\cmidrule{2-8}
 & \multirow[c]{3}{*}{\shortstack{MAPE\\(8.154)}} & 0.050 & 3.794 & 2.945 & \underline{4.109} & 1.539 & 0.573 \\
 &  & 0.075 & \underline{7.512} & 5.314 & 7.246 & 2.682 & 0.798 \\
 &  & 0.100 & \underline{10.563} & 8.387 & 10.269 & 3.968 & 1.049 \\
\cmidrule{2-8}
 & \multirow[c]{3}{*}{\shortstack{RMSE\\(5.934)}} & 0.050 & 1.120 & 0.838 & \underline{1.523} & 0.641 & 0.330 \\
 &  & 0.075 & 3.108 & 1.878 & \underline{3.208} & 1.269 & 0.565 \\
 &  & 0.100 & \underline{5.172} & 3.406 & 5.062 & 2.118 & 0.838 \\
\midrule
\multirow[c]{9}{*}{PEMSD7(L)} & \multirow[c]{3}{*}{\shortstack{MAE\\(3.494)}} & 0.050 & 1.253 & 2.512 & \underline{3.016} & 0.648 & 0.324 \\
 &  & 0.075 & 1.761 & 4.039 & \underline{4.706} & 1.329 & 0.464 \\
 &  & 0.100 & 2.283 & 5.213 & \underline{6.215} & 2.094 & 0.547 \\
\cmidrule{2-8}
 & \multirow[c]{3}{*}{\shortstack{MAPE\\(8.836)}} & 0.050 & 2.867 & 4.628 & \underline{6.058} & 1.330 & 0.629 \\
 &  & 0.075 & 3.422 & 8.003 & \underline{9.190} & 2.353 & 0.849 \\
 &  & 0.100 & 4.255 & 10.487 & \underline{11.722} & 3.474 & 0.986 \\
\cmidrule{2-8}
 & \multirow[c]{3}{*}{\shortstack{RMSE\\(6.308)}} & 0.050 & 1.321 & 2.521 & \underline{3.084} & 0.615 & 0.403 \\
 &  & 0.075 & 1.966 & 4.058 & \underline{4.961} & 1.329 & 0.651 \\
 &  & 0.100 & 2.575 & 5.346 & \underline{6.717} & 2.224 & 0.819 \\
\bottomrule
\end{tabular}
}\vspace{0.5em}
\label{tab:traffic}
\end{table}

\section{Conclusion}
In this work, we study graph injection attacks in a practical black-box setting with access to node-level embeddings and explicit budget constraints on injected node connections. 
Due to its formulation and lack of computational prerequisites, our method is designed to be immediately applicable to any deployed MPNN, without requiring surrogate model training, reinforcement learning, or per-graph iterative optimization. 
Across a range of architectures and tasks, including underexplored graph-level objectives such as graph regression, our approach consistently degrades model performance under tight injection budgets, while maintaining efficiency and ease of deployment.

Empirically, our results follow expected trends: attack effectiveness increases with budget in most cases, and models that rely more heavily on local neighborhood aggregation, especially if they explicitly consume graph topology matrices as-is, are more vulnerable to carefully crafted virtual node injections.
Compared to strong baselines 
our method achieves competitive, and even better performance under the same budget constraints while substantially reducing computational overhead and attack latency, avoiding the nontrivial setup and runtime costs associated with surrogate models or iterative optimizations, which limit their practicality in large-scale or time-sensitive scenarios. 
Our method closes much of this performance gap without these costs, making it better suited for real-world evasion settings. In future, we would like to explore ways to improve the performance of the discrete variants, and applying this to more complex MPNN architectures.

\bibliographystyle{IEEEtran}
\bibliography{bibfile}

\appendix

\subsection{Artifact Availability}
All code, data, and scripts required to reproduce our results are available at the anonymized repository: \url{https://anonymous.4open.science/r/PEANUT/}. The main instructions are included in the main README. This repo will be publicly released upon acceptance.

\subsection{Baseline Code Organization}
\begin{enumerate}
    \item For the AGIA, TDGIA, and ATDGIA baselines, we use the implementation from the UGBA authors\footnote{\url{https://github.com/ventr1c/UGBA}}, as they are all integrated into a single script. Due to certain incompatibilities with modern library versions, certain aspects were modified (example: a small functional change in how the adjacency matrix is updated with the new edges after attacking due to deprecated functionalities related to spare matrices), and the attack script was seeded for reproducibility. The script was also modified to run each experiment $10$ times, and save metrics across runs.
    \item For G2A2C, we use the official implementation\footnote{\url{https://github.com/jumxglhf/G2A2C}}, with the main script \verb|actor_critic.py| modified with added break conditions when the injected node budget is reached. G2A2C specifies how many nodes to inject \textit{per} node in the test set; however, we impose a global limit on the number of injected nodes. 
    \item For the heuristic baselines, we compute the degree and eigenvector centrality via PyTorch, and compute node betweenness centrality using the Python port of \verb|iGraph|.
    \item For the traffic forecasting datasets, 
    \begin{enumerate}
        \item The processed version of the PEMS-BAY dataset is obtained using the official DCRNN repository\footnote{\url{https://github.com/liyaguang/DCRNN}}
        \item The processed version of the PEMSD7(M) and PEMSD7(L) datasets are obtained using the official STGCN repository\footnote{\url{https://github.com/VeritasYin/STGCN_IJCAI-18}}
        \item The data loading is adapted from the official DCRNN repository, and the PEMSD7 datasets are modified to interface with the expected dataloader formats
        \item The STGCN implementation is adapted from a PyTorch implementation of the original model\footnote{\url{https://github.com/hazdzz/STGCN}}
    \end{enumerate}
\end{enumerate}

\subsection{Proof of Lemma \ref{lemma:main_optimization}}\label{appendix:proof_lemma}

\addtocounter{lemma}{-1}
\lemmaone

\begin{proof}
    \begin{align*}
        \norm{\rmB\rmB\tran \rmZ}_F^2 &= \tr \left( (\rmB\rmB\tran \rmZ)\tran(\rmB\rmB\tran \rmZ) \right) \\
        &= \tr \left( \rmZ\tran\rmB\rmB\tran\rmB\rmB\tran\rmZ) \right) \\ 
        &= \tr \left( (\rmB\rmB\tran)^2\rmZ\rmZ\tran) \right)
    \end{align*}
    Since $\rmZ\rmZ\tran$ is real and symmetric, it allows an eigendecomposition $\rmU\bLambda\rmU\tran$. This gives,
    \begin{align*}
        \norm{\rmB\rmB\tran \rmZ}_F^2 &= \tr \left( (\rmB\rmB\tran)^2\rmU\bLambda\rmU\tran \right) \\
        &= \tr \left( (\rmU\tran\rmB\rmB\tran\rmU)^2\bLambda \right) \\
    \end{align*}
    Denoting $\rmY=\rmU\tran\rmB\rmB\tran\rmU$,
    \begin{align*}
        \norm{\rmB\rmB\tran \rmZ}_F^2 &= \tr \left( \rmY^2\bLambda \right) \\
        &= \sum\limits_i \bLambda_{ii} (\rmY^2)_{ii} \\
        &= \sum\limits_i \lambda_i (\rmY^2)_{ii} \\
        &= \sum\limits_i\sum\limits_j \lambda_i \ermY_{ij}^2
    \end{align*}
    The RHS here is a weighted sum of eigenvalues $\lambda_i$ with weights given by $\ermY_{ij}^2$. Since $\rmU$ is orthogonal, 
    \[
        \norm{\rmY}_F = \norm{\rmB\rmB\tran}_F \leq \norm{\rmB}_F^2 \leq \bDelta^2 
    \]
    Or, $\sqrt{\sum\limits_i\sum\limits_j \ermY_{ij}^2} \leq \bDelta^2$. Thus, to maximize the weighted sum given the constraints on the weights, the optimal solution is one where the weights concentrate on the maximum $\lambda_i(=\lambda_1)$, i.e., with the first column of $\rmY$ containing all the weight. The simplest solution which satisfies this is:
    \begin{align*}
        \rmY^* &= \bDelta^2\cdot\rve_1\rve_1\tran \\
        \rmB^*\rmB^{*\top} &= \bDelta^2\cdot\rvu_1\rvu_1\tran
    \end{align*}
    This gives $\rmB^* = \bDelta\cdot\rvu_1\rvv\tran$, where $\rvv$ may be any vector with unit norm ($\rvv$ has no effect on $\rmB^*\rmB^{*\top}$).
\end{proof}
\input{tables/bigtable_gcn_vs_baselines}

\subsection{Equivalent Formulation of $\gL$ for SGC}\label{appendix:sgc_equivalence}
We have the node embeddings $\rmZ$ for a clean graph given by $\rmZ=\rmS^2\rmX\bTheta$. For a perturbation $\rmS_v$, we have the node embeddings given by $\rmS_{p}^2 \rmX_{p} \bTheta$, which is:
\begin{align}
    \rmS_{p}^2 \rmX_{p} \bTheta &=
    \begin{pmatrix}
        \rmS & \rmS_v \\
        \rmS_v\tran & \boldsymbol{0}
    \end{pmatrix} 
    \begin{pmatrix}
        \rmS & \rmS_{v} \\
        \rmS_{v}\tran & \boldsymbol{0}
    \end{pmatrix} \begin{pmatrix} \rmX \\ \tilde{\rmX}\end{pmatrix} \bTheta \\ 
    &=
    \begin{pmatrix}
        \rmS^2 + \rmS_v\rmS_v\tran & \rmS\rmS_v \\
        \rmS_v\tran\rmS & \rmS_v\tran\rmS_v
    \end{pmatrix} \begin{pmatrix} \rmX \\ \tilde{\rmX}\end{pmatrix} \bTheta
\end{align}
Extracting the embeddings of the real nodes, i.e., the first $N$ rows of the above, we have:
\begin{align}
    \rmZ_{p} &= \rmS^2\rmX\bTheta + \rmS_{v}\rmS_{v}\tran\rmX\bTheta + \rmS\rmS_{v}\tilde{\rmX}\bTheta \\\label{eqn:xtilde_effect_on_norm}
    \rmZ_{p} &= \rmZ + \rmS_{v}\rmS_{v}\tran\rmX\bTheta + \rmS\rmS_{v}\tilde{\rmX}\bTheta
\end{align}

This gives us,
\begin{align}
    \norm{\rmZ_{p}-\rmZ}_F^2 &= \norm{\rmS_{v}\rmS_{v}\tran\rmX\bTheta + \rmS\rmS_{v}\tilde{\rmX}\bTheta}_F^2 \\
    &\leq \norm{\rmS_{v}\rmS_{v}\tran\rmX\bTheta}_F^2 + \norm{\rmS\rmS_{v}\tilde{\rmX}\bTheta}_F^2
\end{align}

Specifically, for $\tilde{\rmX}=0$, 
\begin{equation}
    \norm{\rmZ_{p}-\rmZ}_F^2 = \norm{\rmS_{v}\rmS_{v}\tran\rmX\bTheta}_F^2
\end{equation}

\subsection{Hyperparameters}
\subsubsection{Node Classification}
All models trained for NC (\SGC, \GCN, \GIN, and \SAGE) consist of two graph convolution layers of the respective type, with the ReLU activation between them for all models except \SGC. The latent dimension for all datasets is fixed to 16, and there is no Linear readout layer after the second graph convolution. All models are trained using the Adam optimizer \citep{kingma2014adam} with a fixed learning rate of $0.001$, and early stopping with patience $100$.

\subsubsection{Graph-level tasks: Classification and Regression}
All models trained for graph-level tasks (\GCN and \GIN) consist of two graph convolution layers of the respective type, with the ReLU activation between them. This comprises the $\gM^{(agg)}_{\bTheta_2}$ portion of the model. The representations are pooled using Add-pooling, and then fed into the readout module, which is a simple LRL with the output either being a single value (GR) or a set of logits (GC) equal to the number of classes. All models are trained using the Adam optimizer \citep{kingma2014adam} with a starting learning rate of $0.001$, reduced by a factor of $\gamma=0.9$ if the tracked validation metric---F1 score for GC, RMSE for GR---does not improve for $20$ epochs (plateau lr reduction), with the minimum lr set to $1e-4$. We also use early stopping here with the same patience of $100$ epochs.

\subsection{Additional results}\label{appendix:results_additional}

\subsubsection{Performance vs Baselines on the \GCN architecture}\label{appendix:baselines_vs_gcn}
In Table \ref{tab:GCN_baselines}., we re-run the AGIA, TDGIA, and ATDGIA baselines and PEA on the the \GCN architecture. We note essentially the same trends as the effects on the \SGC architecture in Table \ref{tab:SGC_baselines}. 

\begin{figure*}[ht]
    \centering
    \includegraphics[width=\linewidth]{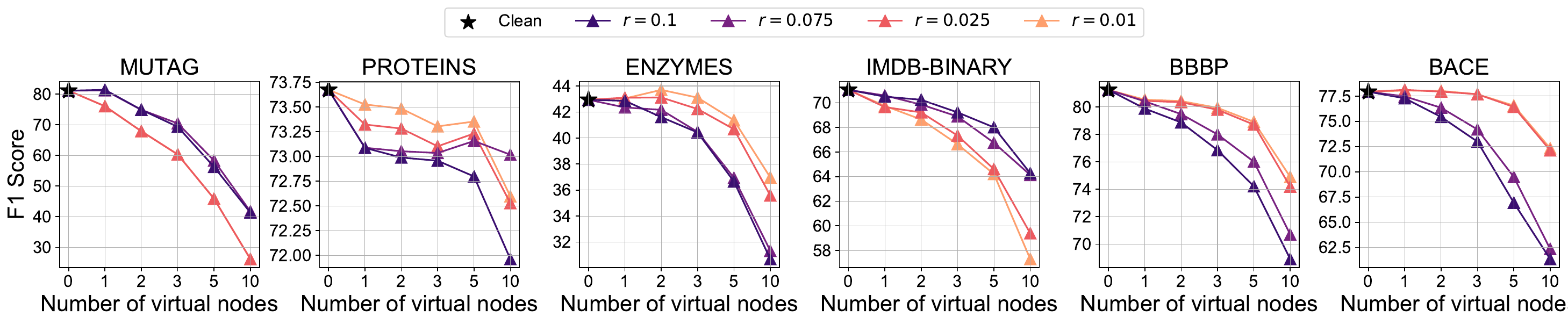}
    \caption{\GIN Classification F1 vs number of virtual nodes on the six classification datasets. All reported numbers have been averaged over the 10-Fold CV.}\vspace{-1em}
    \label{fig:GINA_gc_f1}
\end{figure*}

\begin{figure*}[ht]
    \centering
    \includegraphics[width=\linewidth]{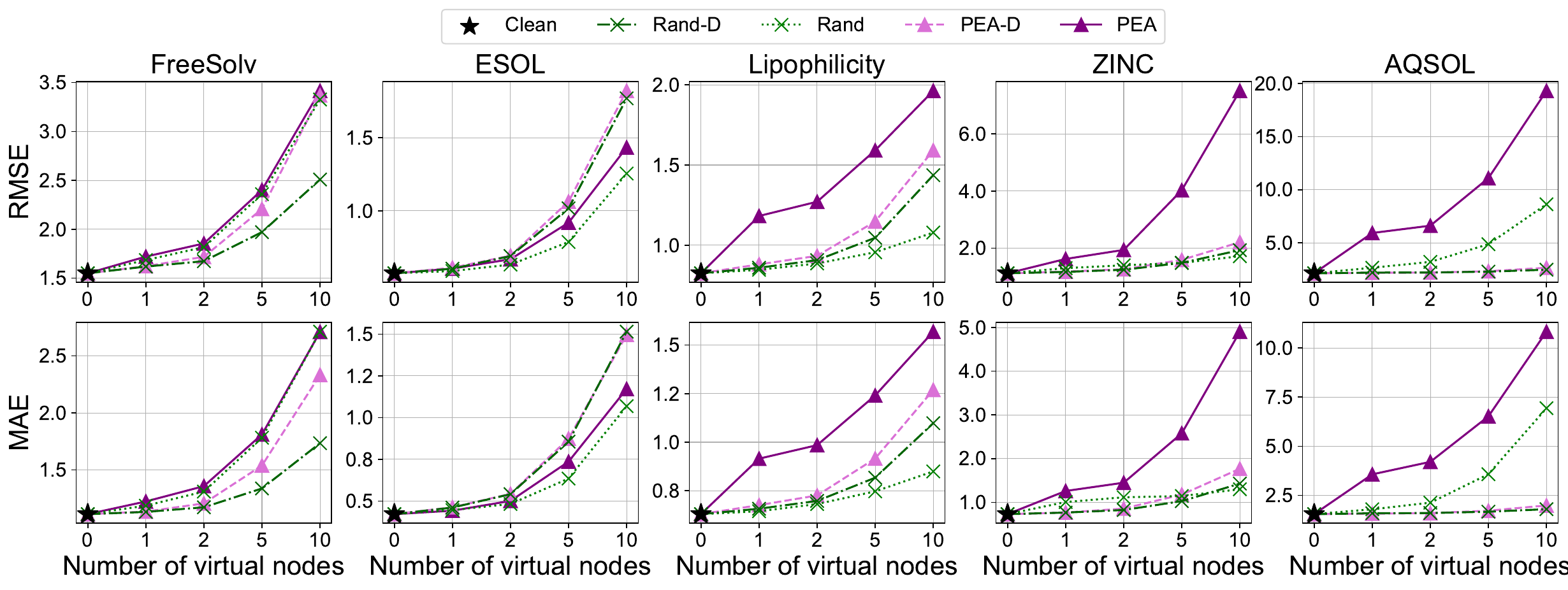}
    \caption{\GIN Regression performance (RMSE and MAE) for \MethodName, \MethodName-D, \rand and \rand-D for the same budget (normalized for the discrete variants), across the five regression datasets.}
    \label{fig:GINA_baselines_mae}
\end{figure*}

\subsubsection{Filter-based Defenses}\label{appendix:defense}
In Tables \ref{tab:filter_defense_cora_with_f1_jacc} and \ref{tab:filter_defense_cora_with_f1_svd}, we report the drops in F1 score in addition to the drop in accuracy. Both metrics follow the same trends, with Jaccard being able to filter out the added perturbations much more than SVD, leading to maintained performance. Although keeping the Jaccard similarity coefficient high ($0.05$) increases the defensive capabilities since more edges are filtered out, the performance on the clean graph is also adversely affected. Compared to a model which has a base accuracy of $86\%$, simply using Jaccard causes performance to drop since many of the original graph edges also get filtered out. 

\subsubsection{Additional Metrics for GC}\label{appendix:gc_additional}
In Figure \ref{fig:GINA_gc_f1}, we plot the F1 score vs number of injected virtual nodes, compared to the Accuracy which was shown in Figure \ref{fig:GINA_gc}.

\begin{table}[H]
\vspace{0.5em}
\caption{Clean accuracy and macro F1 scores followed by their respective decreases for the Jaccard graph defense method on the Cora dataset, using the \SGC architecture.}\label{tab:filter_defense_cora_with_f1_jacc}
\begin{tabular}{rrrrr}
\toprule
\multicolumn{1}{c}{Defense$\rightarrow$} & \multicolumn{2}{c}{Jaccard (0.01)} & \multicolumn{2}{c}{Jaccard (0.05)} \\ \midrule
\multicolumn{1}{c}{$r$} & \multicolumn{1}{c}{Accuracy$\downarrow$} & \multicolumn{1}{c}{F1$\downarrow$} & \multicolumn{1}{c}{Accuracy$\downarrow$} & \multicolumn{1}{c}{F1$\downarrow$} \\\midrule
Clean & 86.08±0.47 & 85.35±0.63 & 83.49±0.47 & 82.45±0.58 \\
0.001 & 0.00±0.06 & 0.00±0.05 & 0.00±0.06 & 0.00±0.06 \\
0.01 & 0.07±0.36 & 0.00±0.42 & 0.00±0.23 & 0.00±0.24 \\
0.05 & 7.06±1.70 & 6.46±1.68 & 0.37±0.31 & 0.27±0.38 \\
0.1 & 37.78±4.42 & 44.42±6.13 & 3.70±1.12 & 3.39±0.98 \\ \bottomrule
\end{tabular}
\end{table}

\begin{table}[H]
\caption{Clean accuracy and macro F1 scores followed by their respective decreases for the SVD graph defense method on the Cora dataset, using the \SGC architecture.}\label{tab:filter_defense_cora_with_f1_svd}
\begin{tabular}{rrrrr}
\toprule
\multicolumn{1}{c}{Defense$\rightarrow$} & \multicolumn{2}{c}{SVD (20)} & \multicolumn{2}{c}{SVD (200)} \\ \midrule
\multicolumn{1}{c}{$r$} & \multicolumn{1}{c}{Accuracy$\downarrow$} & \multicolumn{1}{c}{F1$\downarrow$} & \multicolumn{1}{c}{Accuracy$\downarrow$} & \multicolumn{1}{c}{F1$\downarrow$} \\\midrule
Clean & 75.76±0.49 & 74.22±0.75 & 83.71±0.45 & 83.02±0.52 \\
0.001 & 1.81±0.56 & 2.27±0.69 & 1.60±0.78 & 1.24±0.79 \\
0.01 & 11.73±1.69 & 14.31±1.20 & 5.00±0.78 & 4.28±0.80 \\
0.05 & 34.88±4.14 & 45.98±5.63 & 35.27±4.19 & 42.27±5.25 \\
0.1 & 46.98±0.90 & 65.73±1.42 & 54.91±0.55 & 74.56±1.03 \\ \bottomrule
\end{tabular}
\end{table}

\subsubsection{Ablation Study on \SAGE and \GIN}\label{appendix:sage_ablation}
Continuing from Section \ref{sec:sage_discussion}, we report the drops in Accuracy and macro F1 scores across four values of $r$ for all three NC datasets in Tables \ref{tab:sage_ablation_all1}-\ref{tab:sage_ablation_all3}.

\begin{table}[h]
\caption{\SAGE/\GIN ablation on the Cora dataset}\label{tab:sage_ablation_all1}
\resizebox{\linewidth}{!}{
\begin{tabular}{ccccccc}
\toprule
\multirow{2}{*}{Layer} & \multirow{2}{*}{Aggregation} & \multirow{2}{*}{Metric} & \multicolumn{4}{c}{$r$} \\\cmidrule(lr){4-7}
 &  & & 0.001 & 0.01 & 0.05 & 0.1 \\\midrule
\multirow{4}{*}{SAGE} & \multirow{2}{*}{Mean} & Accuracy$\downarrow$ & 0.04±0.09 & 0.10±0.25 & 3.09±0.65 & 7.73±1.35 \\
 &  & F1$\downarrow$ & 0.06±0.09 & 0.14±0.27 & 3.04±0.64 & 7.64±1.23 \\\cmidrule(lr){2-7}
 & \multirow{2}{*}{Sum} & Accuracy$\downarrow$ & 0.02±0.16 & 4.31±1.88 & 49.26±7.81 & 54.25±5.67 \\
 &  & F1$\downarrow$ & -0.02±0.20 & 4.64±3.02 & 68.47±8.14 & 74.61±6.67 \\\midrule
\multirow{4}{*}{GIN} & \multirow{2}{*}{Sum} & Accuracy$\downarrow$ & 0.03±0.23 & 12.39±8.18 & 51.74±3.79 & 54.62±4.50 \\
 &  & F1$\downarrow$ & 0.07±0.31 & 16.54±12.40 & 72.28±3.69 & 74.33±2.93 \\\cmidrule(lr){2-7}
 & \multirow{2}{*}{Mean} & Accuracy$\downarrow$ & 0.09±0.06 & 0.81±0.36 & 4.70±0.89 & 9.40±1.75 \\
 &  & F1$\downarrow$ & 0.11±0.09 & 0.91±0.38 & 5.42±1.12 & 10.82±1.90\\\bottomrule
\end{tabular}
}
\end{table}

\begin{table}[h]
\caption{\SAGE/\GIN ablation on the Citeseer dataset}\label{tab:sage_ablation_all2}
\resizebox{\linewidth}{!}{
\begin{tabular}{ccccccc}
\toprule
\multirow{2}{*}{Layer} & \multirow{2}{*}{Aggregation} & \multirow{2}{*}{Metric} & \multicolumn{4}{c}{$r$} \\\cmidrule(lr){4-7}
 &  & & 0.001 & 0.01 & 0.05 & 0.1 \\\midrule
\multirow{4}{*}{SAGE} & \multirow{2}{*}{Mean} & Accuracy$\downarrow$ & 0.01±0.08 & 0.09±0.20 & 0.65±0.35 & 1.88±0.43 \\
 & & F1$\downarrow$ & 0.00±0.12 & 0.14±0.24 & 0.66±0.57 & 2.13±0.66 \\\cmidrule(lr){2-7}
 & \multirow{2}{*}{Sum} & Accuracy$\downarrow$ & -0.04±0.18 & 1.67±2.16 & 34.56±14.10 & 43.26±12.41 \\
 & & F1$\downarrow$ & 0.06±0.23 & 2.41±2.36 & 44.31±15.50 & 53.31±13.18 \\\midrule
\multirow{4}{*}{GIN} & \multirow{2}{*}{Sum} & Accuracy$\downarrow$ & 0.20±0.29 & 22.30±9.79 & 41.21±10.76 & 46.83±8.23 \\
 & & F1$\downarrow$ & 0.27±0.38 & 26.49±10.95 & 50.01±10.33 & 55.61±8.50 \\\cmidrule(lr){2-7}
 & \multirow{2}{*}{Mean} & Accuracy$\downarrow$ & 0.01±0.05 & 0.32±0.34 & 1.79±1.07 & 3.99±1.51 \\
 & & F1$\downarrow$ & 0.00±0.05 & 0.23±0.40 & 1.78±1.28 & 3.82±1.61\\\bottomrule
\end{tabular}
}\vspace{0.5em}
\end{table}

\begin{table}[h]
\caption{\SAGE/\GIN ablation on the Pubmed dataset}\label{tab:sage_ablation_all3}
\resizebox{\linewidth}{!}{
\begin{tabular}{ccccccc}
\toprule
\multirow{2}{*}{Layer} & \multirow{2}{*}{Aggregation} & \multirow{2}{*}{Metric} & \multicolumn{4}{c}{$r$} \\\cmidrule(lr){4-7}
 &  & & 0.001 & 0.01 & 0.05 & 0.1 \\\midrule
\multirow{4}{*}{SAGE} & \multirow{2}{*}{Mean} & Accuracy$\downarrow$ & -0.05±0.05 & -0.29±0.13 & 0.79±0.22 & 2.02±0.41 \\
 & & F1$\downarrow$ & -0.05±0.05 & -0.31±0.13 & 0.82±0.26 & 2.22±0.61 \\\cmidrule(lr){2-7}
 & \multirow{2}{*}{Sum} & Accuracy$\downarrow$ & 0.15±0.18 & 28.68±8.91 & 44.94±1.16 & 45.27±0.30 \\
 & & F1$\downarrow$ & 0.31±0.21 & 44.32±9.31 & 64.75±1.91 & 65.35±0.38 \\\midrule
\multirow{4}{*}{GIN} & \multirow{2}{*}{Sum} & Accuracy$\downarrow$ & 0.90±1.17 & 30.38±11.05 & 40.04±9.28 & 40.45±8.72 \\
 & & F1$\downarrow$ & 1.20±1.62 & 43.60±13.75 & 58.70±11.33 & 59.30±10.74 \\\cmidrule(lr){2-7}
 & \multirow{2}{*}{Mean} & Accuracy$\downarrow$ & 0.19±0.17 & 3.05±2.50 & 11.28±4.52 & 14.29±5.43 \\
 & & F1$\downarrow$ & 0.15±0.20 & 3.71±3.97 & 16.65±9.45 & 21.44±9.90\\\bottomrule
\end{tabular}
}\vspace{0.5em}
\end{table}

\subsubsection{Additional Metrics for GR}\label{appendix:gr_additional}
In Figure \ref{fig:GINA_baselines_mae}, we also observe similar trends as Figure \ref{fig:gr_heuristics}, on the MAE metric.

\subsubsection{Exploring Alternative Methods of Choosing $\tilde{\rmX}$}\label{appendix:xtilde}

\input{tables/heuristics_vs_xtilde}
In Table \ref{tab:heuristics_vs_xtilde_full}, we provide a more thorough evaluation of different modes of choosing $\tilde{\rmX}$, with evaluations across the four heuristic baselines. Using the eigenvector centrality of the adjacency matrix seems to be a weak baseline overall, with poor attack performance across all budgets, though choosing $\tilde{\rmX}$ as (scaled) random uniform seems to help. In these results, \MethodName boasts superior performance in most comparisons, only falling behind on \SAGE, which is also overturned by using the rand mode for $\tilde{\rmX}$.

\subsubsection{Performance of Heuristic Baselines across all MPNN architectures}\label{appendix:heuristics}
Similar to Figure \ref{fig:sgc_heuristics_nof1} in the text, in Figure \ref{fig:sgc_heuristics_full}, we plot the accuracy and norm $\norm{\rmZ_p-\rmZ}_F$ trends for all model types across the three datasets. Note that the models where we can perturb $\rmS$ directly (S-\SGC and S-\GCN) suffer significantly larger $\gL(\rmS_v)$ compared to the others.

\begin{figure*}[ht]
    \centering
    \subfloat[Cora]{\includegraphics[width=\linewidth]{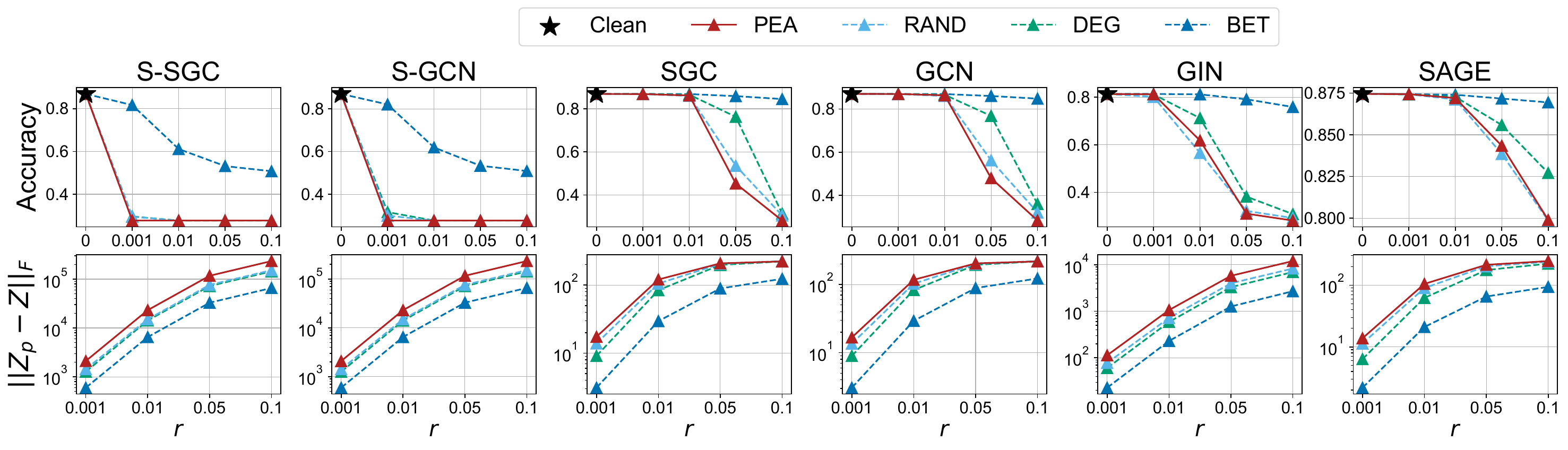}
    }
    \vspace{0.5em}
    \subfloat[Citeseer]{\includegraphics[width=\linewidth]{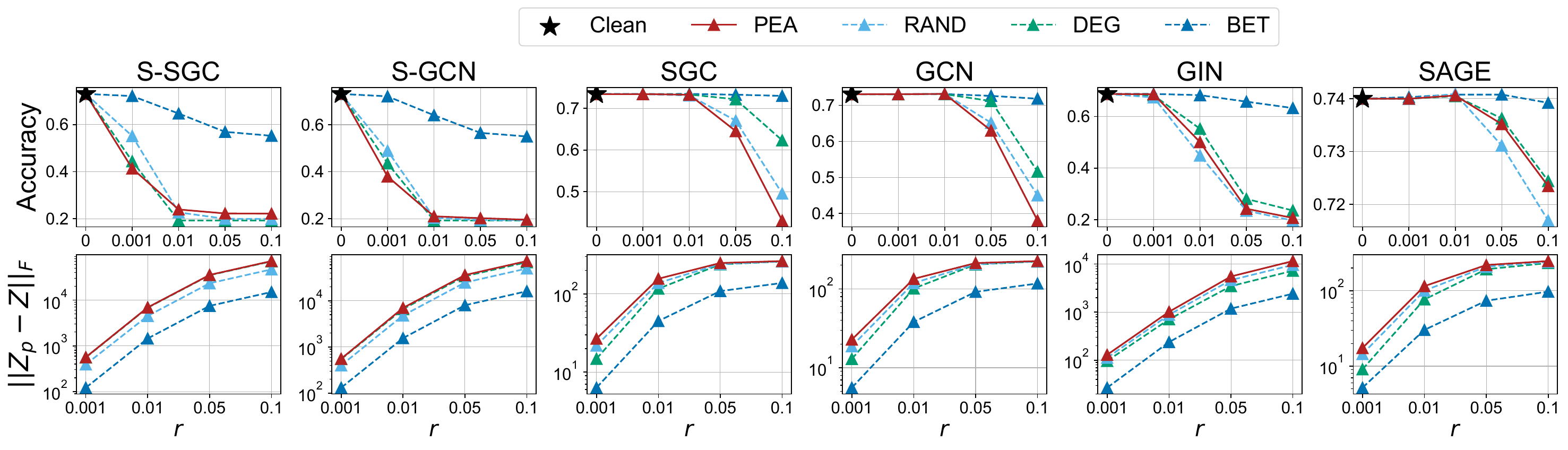}}
    \vspace{0.5em}
    \subfloat[Pubmed]{\includegraphics[width=\linewidth]{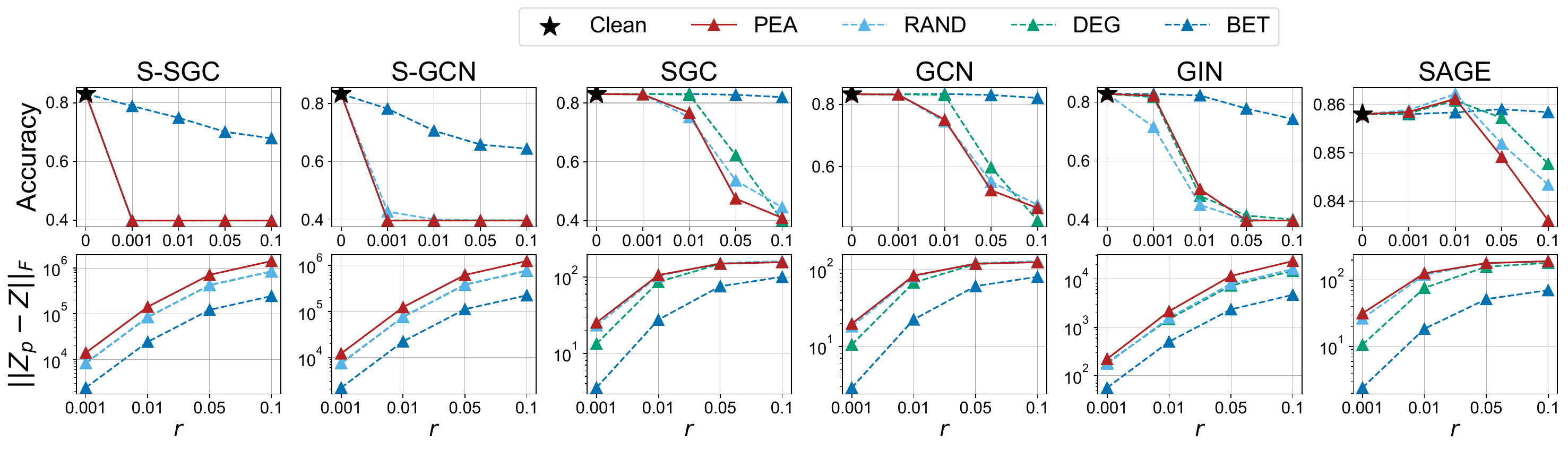}}
    \caption{Performance of  \MethodName and the heuristic baselines on all three NC datasets over different perturbation budgets $r$}
    \label{fig:sgc_heuristics_full}
\end{figure*}

\end{document}

%% file: math_commands.tex
\usepackage{amsmath,amsfonts,bm}

\newcommand{\tran}{^\top}

\def\eqref#1{equation~\ref{#1}}

\def\1{\bm{1}}

\def\rve{{\mathbf{e}}}

\def\rvu{{\mathbf{i}}}

\def\rvu{{\mathbf{u}}}
\def\rvv{{\mathbf{v}}}

\def\rmA{{\mathbf{A}}}
\def\rmB{{\mathbf{B}}}

\def\rmD{{\mathbf{D}}}

\def\rmH{{\mathbf{H}}}
\def\rmI{{\mathbf{I}}}

\def\rmL{{\mathbf{L}}}

\def\rmS{{\mathbf{S}}}

\def\rmU{{\mathbf{U}}}
\def\rmV{{\mathbf{V}}}
\def\rmW{{\mathbf{W}}}
\def\rmX{{\mathbf{X}}}
\def\rmY{{\mathbf{Y}}}
\def\rmZ{{\mathbf{Z}}}

\def\ermY{{\textnormal{Y}}}

\DeclareMathAlphabet{\mathsfit}{\encodingdefault}{\sfdefault}{m}{sl}
\SetMathAlphabet{\mathsfit}{bold}{\encodingdefault}{\sfdefault}{bx}{n}

\def\gC{{\mathcal{C}}}

\def\gE{{\mathcal{E}}}

\def\gG{{\mathcal{G}}}

\def\gL{{\mathcal{L}}}
\def\gM{{\mathcal{M}}}

\def\gV{{\mathcal{V}}}

\def\sR{{\mathbb{R}}}

\DeclareMathOperator{\sign}{sign}

%% file: tables/bigtable_sgc_vs_baselines.tex
\begin{table*}
\caption{Results of \MethodName and the baselines on NC using the \SGC architecture as the defending model. The \emph{decrease} in accuracy and (macro) F1 score are reported for each method, with statistics computed over $10$ runs for all methods except G2A2C, which is over $5$ runs ($3$ for Pubmed) due to very high execution times.}
\resizebox{\textwidth}{!}{
\begin{tabular}{lrrrrrrrrrrrrrr}
\toprule 
 \multirow{2}{*}{\vspace{-4pt}Dataset}  & \multicolumn{1}{c}{\multirow{2}{*}{\vspace{-4pt}$r$}} & \multicolumn{1}{c}{\multirow{2}{*}{\vspace{-4pt}$n_v$}} & \multicolumn{2}{c}{AGIA} & \multicolumn{2}{c}{TDGIA} & \multicolumn{2}{c}{ATDGIA} & \multicolumn{2}{c}{G2A2C} & \multicolumn{2}{c}{PEA} & \multicolumn{2}{c}{PEA (Acting on $\rmS$)} \\ \cmidrule(lr){4-5}\cmidrule(lr){6-7}\cmidrule(lr){8-9}\cmidrule(lr){10-11}\cmidrule(lr){12-13}\cmidrule(lr){14-15}
 &  &  & \multicolumn{1}{c}{Accuracy$\downarrow$} & \multicolumn{1}{c}{F1$\downarrow$} & \multicolumn{1}{c}{Accuracy$\downarrow$} & \multicolumn{1}{c}{F1$\downarrow$} & \multicolumn{1}{c}{Accuracy$\downarrow$} & \multicolumn{1}{c}{F1$\downarrow$} & \multicolumn{1}{c}{Accuracy$\downarrow$} & \multicolumn{1}{c}{F1$\downarrow$} & \multicolumn{1}{c}{Accuracy$\downarrow$} & \multicolumn{1}{c}{F1$\downarrow$} & \multicolumn{1}{c}{Accuracy$\downarrow$} & \multicolumn{1}{c}{F1$\downarrow$} \\
\midrule
\multirow[c]{4}{*}{Cora} & 0.001 & 2 & $0.00${\scriptsize$\pm0.00$} & $0.00${\scriptsize$\pm0.00$} & $0.00${\scriptsize$\pm0.00$} & $0.00${\scriptsize$\pm0.00$} & $0.10${\scriptsize$\pm0.12$} & $0.09${\scriptsize$\pm0.13$} & \underline{$0.21${\scriptsize$\pm0.05$}} & \underline{$0.20${\scriptsize$\pm0.06$}} & $0.07${\scriptsize$\pm0.09$} & $0.06${\scriptsize$\pm0.10$} & $59.20${\scriptsize$\pm0.29$} & $80.10${\scriptsize$\pm0.32$} \\
 & 0.01 & 27 & $1.34${\scriptsize$\pm0.14$} & $1.43${\scriptsize$\pm0.26$} & $1.45${\scriptsize$\pm0.33$} & $1.50${\scriptsize$\pm0.34$} & \underline{$2.23${\scriptsize$\pm0.37$}} & \underline{$2.65${\scriptsize$\pm0.39$}} & $2.19${\scriptsize$\pm0.16$} & $2.11${\scriptsize$\pm0.25$} & $0.73${\scriptsize$\pm0.31$} & $0.54${\scriptsize$\pm0.29$} & $59.20${\scriptsize$\pm0.29$} & $80.10${\scriptsize$\pm0.32$} \\
 & 0.05 & 135 & $16.77${\scriptsize$\pm0.74$} & $17.25${\scriptsize$\pm0.63$} & $16.03${\scriptsize$\pm1.01$} & $16.74${\scriptsize$\pm1.19$} & $28.30${\scriptsize$\pm0.89$} & $30.69${\scriptsize$\pm0.89$} & $9.78${\scriptsize$\pm0.30$} & $9.68${\scriptsize$\pm0.37$} & \underline{$41.58${\scriptsize$\pm3.75$}} & \underline{$47.60${\scriptsize$\pm5.05$}} & $59.20${\scriptsize$\pm0.29$} & $80.10${\scriptsize$\pm0.32$} \\
 & 0.1 & 270 & $29.14${\scriptsize$\pm0.65$} & $30.87${\scriptsize$\pm0.62$} & $29.50${\scriptsize$\pm0.88$} & $30.94${\scriptsize$\pm0.71$} & $37.92${\scriptsize$\pm1.46$} & $39.21${\scriptsize$\pm1.24$} & $19.65${\scriptsize$\pm0.42$} & $19.93${\scriptsize$\pm0.42$} & \underline{$58.74${\scriptsize$\pm0.37$}} & \underline{$79.12${\scriptsize$\pm0.55$}} & $59.20${\scriptsize$\pm0.29$} & $80.10${\scriptsize$\pm0.32$} \\
\midrule 
\multirow[c]{4}{*}{Citeseer} & 0.001 & 3 & $0.00${\scriptsize$\pm0.00$} & $0.00${\scriptsize$\pm0.00$} & $0.00${\scriptsize$\pm0.00$} & $0.00${\scriptsize$\pm0.00$} & $0.00${\scriptsize$\pm0.00$} & $0.01${\scriptsize$\pm0.01$} & \underline{$0.20${\scriptsize$\pm0.06$}} & \underline{$0.28${\scriptsize$\pm0.10$}} & $0.01${\scriptsize$\pm0.03$} & $0.02${\scriptsize$\pm0.05$} & $31.67${\scriptsize$\pm4.84$} & $42.42${\scriptsize$\pm6.19$} \\
 & 0.01 & 33 & $1.62${\scriptsize$\pm0.12$} & $1.51${\scriptsize$\pm0.12$} & $1.46${\scriptsize$\pm0.20$} & $1.37${\scriptsize$\pm0.18$} & $0.77${\scriptsize$\pm0.20$} & $1.09${\scriptsize$\pm0.37$} & \underline{$1.95${\scriptsize$\pm0.13$}} & \underline{$2.52${\scriptsize$\pm0.25$}} & $0.15${\scriptsize$\pm0.28$} & $0.38${\scriptsize$\pm0.32$} & $49.07${\scriptsize$\pm5.73$} & $60.93${\scriptsize$\pm4.64$} \\
 & 0.05 & 166 & $10.11${\scriptsize$\pm0.60$} & $9.94${\scriptsize$\pm0.80$} & $9.58${\scriptsize$\pm0.63$} & $9.36${\scriptsize$\pm0.63$} & \underline{$10.29${\scriptsize$\pm0.79$}} & \underline{$10.09${\scriptsize$\pm1.21$}} & $5.23${\scriptsize$\pm2.38$} & $7.63${\scriptsize$\pm2.68$} & $8.88${\scriptsize$\pm3.40$} & $13.06${\scriptsize$\pm3.44$} & $50.76${\scriptsize$\pm4.32$} & $62.56${\scriptsize$\pm3.54$} \\
 & 0.1 & 332 & $16.00${\scriptsize$\pm1.06$} & $15.33${\scriptsize$\pm1.21$} & $15.79${\scriptsize$\pm1.01$} & $15.20${\scriptsize$\pm1.18$} & $25.80${\scriptsize$\pm0.84$} & $23.32${\scriptsize$\pm1.08$} & $4.74${\scriptsize$\pm1.56$} & $7.48${\scriptsize$\pm2.09$} & \underline{$30.34${\scriptsize$\pm7.38$}} & \underline{$36.66${\scriptsize$\pm6.79$}} & $50.82${\scriptsize$\pm4.12$} & $62.44${\scriptsize$\pm3.88$} \\
\midrule 
\multirow[c]{4}{*}{Pubmed} & 0.001 & 19 & $0.00${\scriptsize$\pm0.00$} & $0.00${\scriptsize$\pm0.00$} & $0.01${\scriptsize$\pm0.01$} & $0.01${\scriptsize$\pm0.01$} & $0.09${\scriptsize$\pm0.04$} & $0.10${\scriptsize$\pm0.04$} & $0.18${\scriptsize$\pm0.02$} & \underline{$0.18${\scriptsize$\pm0.02$}} & \underline{$0.14${\scriptsize$\pm0.05$}} & $0.23${\scriptsize$\pm0.05$} & $43.18${\scriptsize$\pm0.48$} & $63.44${\scriptsize$\pm0.45$} \\
 & 0.01 & 197 & $0.68${\scriptsize$\pm0.09$} & $0.68${\scriptsize$\pm0.10$} & $0.68${\scriptsize$\pm0.06$} & $0.69${\scriptsize$\pm0.06$} & $0.76${\scriptsize$\pm0.08$} & $0.78${\scriptsize$\pm0.06$} & \underline{$13.26${\scriptsize$\pm11.21$}} &\underline{$14.22${\scriptsize$\pm12.18$}} & $6.33${\scriptsize$\pm0.66$} & $11.32${\scriptsize$\pm1.44$} & $43.18${\scriptsize$\pm0.48$} & $63.44${\scriptsize$\pm0.45$} \\
 & 0.05 & 985 & $5.39${\scriptsize$\pm0.16$} & $5.49${\scriptsize$\pm0.18$} & $5.47${\scriptsize$\pm0.23$} & $5.58${\scriptsize$\pm0.20$} & $11.44${\scriptsize$\pm0.52$} & $11.72${\scriptsize$\pm0.54$} & $20.12${\scriptsize$\pm6.28$} & $35.54${\scriptsize$\pm2.27$} & \underline{$35.54${\scriptsize$\pm2.27$}} & \underline{$51.25${\scriptsize$\pm2.93$}} & $43.18${\scriptsize$\pm0.48$} & $63.44${\scriptsize$\pm0.45$} \\
 & 0.1 & 1971 & $7.21${\scriptsize$\pm0.25$} & $7.38${\scriptsize$\pm0.29$} & $7.28${\scriptsize$\pm0.22$} & $7.47${\scriptsize$\pm0.25$} & $18.59${\scriptsize$\pm1.05$} & $18.88${\scriptsize$\pm1.07$} & $29.49${\scriptsize$\pm1.19$} & $42.12${\scriptsize$\pm0.87$}	 & \underline{$42.12${\scriptsize$\pm0.87$}} & \underline{$61.45${\scriptsize$\pm1.22$}} & $43.18${\scriptsize$\pm0.48$} & $63.44${\scriptsize$\pm0.45$} \\
\bottomrule
\end{tabular}
}\vspace{0.5em}
\label{tab:SGC_baselines}
\end{table*}

%% file: tables/bigtable_sgc_vs_heuristics.tex
\begin{table*}
\caption{Results of \MethodName and the heuristic baselines on NC using the \SGC architecture as the defending model. The \emph{decrease} in accuracy and (macro) F1 score are reported for each method, with statistics computed over $10$ runs.}%
\resizebox{\textwidth}{!}{
\begin{tabular}{lrrrrrrrrrrrrrr}
\toprule 
 \multirow{2}{*}{\vspace{-4pt}Dataset}  & \multicolumn{1}{c}{\multirow{2}{*}{\vspace{-4pt}$r$}} & \multicolumn{1}{c}{\multirow{2}{*}{\vspace{-4pt}$n_v$}} & \multicolumn{2}{c}{Clean} & \multicolumn{2}{c}{RAND} & \multicolumn{2}{c}{DEG} & \multicolumn{2}{c}{BET} & \multicolumn{2}{c}{CENT} & \multicolumn{2}{c}{PEA} \\ \cmidrule(lr){4-5}\cmidrule(lr){6-7}\cmidrule(lr){8-9}\cmidrule(lr){10-11}\cmidrule(lr){12-13}\cmidrule(lr){14-15}
 &  &  & \multicolumn{1}{c}{Accuracy} & \multicolumn{1}{c}{F1} & \multicolumn{1}{c}{Accuracy$\downarrow$} & \multicolumn{1}{c}{F1$\downarrow$} & \multicolumn{1}{c}{Accuracy$\downarrow$} & \multicolumn{1}{c}{F1$\downarrow$} & \multicolumn{1}{c}{Accuracy$\downarrow$} & \multicolumn{1}{c}{F1$\downarrow$} & \multicolumn{1}{c}{Accuracy$\downarrow$} & \multicolumn{1}{c}{F1$\downarrow$} & \multicolumn{1}{c}{Accuracy$\downarrow$} & \multicolumn{1}{c}{F1$\downarrow$} \\
\midrule
\multirow[c]{4}{*}{Cora} & 0.001 & 2 & $86.90${\scriptsize$\pm0.29$} & $86.30${\scriptsize$\pm0.32$} & $0.07${\scriptsize$\pm0.28$} & $0.06${\scriptsize$\pm0.32$} & $0.06${\scriptsize$\pm0.35$} & $0.04${\scriptsize$\pm0.40$} & $0.01${\scriptsize$\pm0.30$} & $-0.00${\scriptsize$\pm0.34$} & $0.01${\scriptsize$\pm0.28$} & $0.00${\scriptsize$\pm0.31$} & \underline{$0.07${\scriptsize$\pm0.32$}} & \underline{$0.06${\scriptsize$\pm0.35$}} \\
 & 0.01 & 27 & $86.90${\scriptsize$\pm0.29$} & $86.30${\scriptsize$\pm0.32$} & \underline{$0.92${\scriptsize$\pm0.24$}} & \underline{$0.76${\scriptsize$\pm0.44$}} & $0.35${\scriptsize$\pm0.26$} & $0.22${\scriptsize$\pm0.38$} & $0.07${\scriptsize$\pm0.34$} & $0.01${\scriptsize$\pm0.38$} & $-0.06${\scriptsize$\pm0.28$} & $-0.11${\scriptsize$\pm0.30$} & $0.73${\scriptsize$\pm0.32$} & $0.54${\scriptsize$\pm0.42$} \\
 & 0.05 & 135 & $86.90${\scriptsize$\pm0.29$} & $86.30${\scriptsize$\pm0.32$} & $33.32${\scriptsize$\pm2.77$} & $36.68${\scriptsize$\pm3.21$} & $10.63${\scriptsize$\pm2.23$} & $10.72${\scriptsize$\pm2.27$} & $1.03${\scriptsize$\pm0.26$} & $0.74${\scriptsize$\pm0.35$} & $-0.01${\scriptsize$\pm0.32$} & $-0.05${\scriptsize$\pm0.38$} & \underline{$41.58${\scriptsize$\pm3.74$}} & \underline{$47.60${\scriptsize$\pm5.04$}} \\
 & 0.1 & 270 & $86.90${\scriptsize$\pm0.29$} & $86.30${\scriptsize$\pm0.32$} & $56.39${\scriptsize$\pm0.66$} & $74.47${\scriptsize$\pm1.27$} & $55.52${\scriptsize$\pm1.92$} & $73.04${\scriptsize$\pm3.22$} & $2.39${\scriptsize$\pm0.27$} & $1.80${\scriptsize$\pm0.30$} & $0.44${\scriptsize$\pm0.37$} & $0.51${\scriptsize$\pm0.45$} & \underline{$58.74${\scriptsize$\pm0.15$}} & \underline{$79.12${\scriptsize$\pm0.35$}} \\
\midrule
\multirow[c]{4}{*}{Citeseer} & 0.001 & 3 & $73.42${\scriptsize$\pm0.44$} & $69.54${\scriptsize$\pm0.40$} & $0.00${\scriptsize$\pm0.44$} & $0.01${\scriptsize$\pm0.40$} & $0.01${\scriptsize$\pm0.45$} & $0.02${\scriptsize$\pm0.40$} & $0.00${\scriptsize$\pm0.46$} & $0.00${\scriptsize$\pm0.42$} & $-0.02${\scriptsize$\pm0.44$} & $-0.02${\scriptsize$\pm0.40$} & \underline{$0.01${\scriptsize$\pm0.44$}} & \underline{$0.02${\scriptsize$\pm0.40$}} \\
 & 0.01 & 33 & $73.42${\scriptsize$\pm0.44$} & $69.54${\scriptsize$\pm0.40$} & \underline{$0.32${\scriptsize$\pm0.43$}} & \underline{$0.75${\scriptsize$\pm0.32$}} & $0.10${\scriptsize$\pm0.45$} & $0.20${\scriptsize$\pm0.40$} & $-0.04${\scriptsize$\pm0.49$} & $-0.02${\scriptsize$\pm0.44$} & $-0.13${\scriptsize$\pm0.43$} & $-0.11${\scriptsize$\pm0.40$} & $0.15${\scriptsize$\pm0.51$} & $0.38${\scriptsize$\pm0.41$} \\
 & 0.05 & 166 & $73.42${\scriptsize$\pm0.44$} & $69.54${\scriptsize$\pm0.40$} & $6.37${\scriptsize$\pm2.14$} & $10.01${\scriptsize$\pm1.99$} & $1.25${\scriptsize$\pm0.44$} & $2.92${\scriptsize$\pm0.44$} & $0.15${\scriptsize$\pm0.47$} & $0.26${\scriptsize$\pm0.41$} & $0.10${\scriptsize$\pm0.45$} & $0.08${\scriptsize$\pm0.41$} & \underline{$8.88${\scriptsize$\pm3.47$}} & \underline{$13.06${\scriptsize$\pm3.50$}} \\
 & 0.1 & 332 & $73.42${\scriptsize$\pm0.44$} & $69.54${\scriptsize$\pm0.40$} & $23.86${\scriptsize$\pm5.88$} & $28.54${\scriptsize$\pm5.27$} & $11.10${\scriptsize$\pm3.44$} & $17.66${\scriptsize$\pm4.60$} & $0.43${\scriptsize$\pm0.44$} & $0.51${\scriptsize$\pm0.41$} & $-0.01${\scriptsize$\pm0.43$} & $-0.01${\scriptsize$\pm0.40$} & \underline{$30.34${\scriptsize$\pm7.32$}} & \underline{$36.66${\scriptsize$\pm6.77$}} \\
\midrule
\multirow[c]{4}{*}{Pubmed} & 0.001 & 19 & $83.02${\scriptsize$\pm0.48$} & $82.43${\scriptsize$\pm0.45$} & \underline{$0.18${\scriptsize$\pm0.46$}} & \underline{$0.29${\scriptsize$\pm0.43$}} & $-0.01${\scriptsize$\pm0.48$} & $0.01${\scriptsize$\pm0.46$} & $-0.01${\scriptsize$\pm0.49$} & $-0.01${\scriptsize$\pm0.47$} & $-0.01${\scriptsize$\pm0.48$} & $-0.01${\scriptsize$\pm0.46$} & $0.14${\scriptsize$\pm0.47$} & $0.23${\scriptsize$\pm0.44$} \\
 & 0.01 & 197 & $83.02${\scriptsize$\pm0.48$} & $82.43${\scriptsize$\pm0.45$} & \underline{$7.86${\scriptsize$\pm0.72$}} & \underline{$14.45${\scriptsize$\pm1.12$}} & $0.21${\scriptsize$\pm0.48$} & $0.56${\scriptsize$\pm0.48$} & $-0.05${\scriptsize$\pm0.55$} & $-0.07${\scriptsize$\pm0.53$} & $-0.02${\scriptsize$\pm0.48$} & $-0.01${\scriptsize$\pm0.45$} & $6.33${\scriptsize$\pm1.10$} & $11.32${\scriptsize$\pm1.86$} \\
 & 0.05 & 985 & $83.02${\scriptsize$\pm0.48$} & $82.43${\scriptsize$\pm0.45$} & $29.39${\scriptsize$\pm4.04$} & $43.95${\scriptsize$\pm4.07$} & $20.82${\scriptsize$\pm2.69$} & $35.92${\scriptsize$\pm2.19$} & $0.25${\scriptsize$\pm0.57$} & $0.36${\scriptsize$\pm0.58$} & $0.01${\scriptsize$\pm0.53$} & $0.03${\scriptsize$\pm0.50$} & \underline{$35.54${\scriptsize$\pm1.94$}} & \underline{$51.25${\scriptsize$\pm2.64$}} \\
 & 0.1 & 1971 & $83.02${\scriptsize$\pm0.48$} & $82.43${\scriptsize$\pm0.45$} & $38.48${\scriptsize$\pm1.79$} & $55.64${\scriptsize$\pm2.65$} & $43.00${\scriptsize$\pm0.24$} & $63.12${\scriptsize$\pm0.44$} & $1.04${\scriptsize$\pm0.56$} & $1.47${\scriptsize$\pm0.60$} & $0.15${\scriptsize$\pm0.49$} & $0.16${\scriptsize$\pm0.45$} & \underline{$42.12${\scriptsize$\pm0.51$}} & \underline{$61.45${\scriptsize$\pm0.92$}} \\
\bottomrule
\end{tabular}
}\vspace{0.5em}
\label{tab:SGC_heuristics}
\end{table*}

%% file: tables/nc_arch_xtilde_subfloat.tex
\begin{figure*}[ht]
\begin{minipage}[t]{0.49\textwidth}
\centering
\subfloat[Performance of PEA on NC across MPNN architectures.]{
\resizebox{\linewidth}{!}{
\begin{tabular}{llrrcrr}
\toprule
Dataset & Layer & \multicolumn{1}{c}{Clean Acc} & \multicolumn{1}{c}{Clean F1} & $r$ & \multicolumn{1}{c}{Accuracy$\downarrow$} & \multicolumn{1}{c}{F1$\downarrow$} \\ \midrule
\multirow{12}{*}{\vspace{-2em}Cora} & \multirow{2}{*}{S-SGC} & \multirow{2}{*}{$86.90${\small$\pm0.29$}} & \multirow{2}{*}{$86.30${\small$\pm0.32$}} & 0.05 & $59.20${\small$\pm0.00$} & $80.10${\small$\pm0.00$} \\
 &  &  &  & 0.1 & $59.20${\small$\pm0.00$} & $80.10${\small$\pm0.00$} \\\cmidrule(lr){2-7}
 & \multirow{2}{*}{S-GCN} & \multirow{2}{*}{$86.98${\small$\pm0.28$}} & \multirow{2}{*}{$86.40${\small$\pm0.32$}} & 0.05 & $59.28${\small$\pm0.00$} & $80.20${\small$\pm0.00$} \\
 &  &  &  & 0.1 & $59.28${\small$\pm0.00$} & $80.20${\small$\pm0.00$} \\\cmidrule(lr){2-7}
 & \multirow{2}{*}{SGC} & \multirow{2}{*}{$86.90${\small$\pm0.29$}} & \multirow{2}{*}{$86.30${\small$\pm0.32$}} & 0.05 & $41.58${\small$\pm3.74$} & $47.60${\small$\pm5.04$} \\
 &  &  &  & 0.1 & $58.74${\small$\pm0.15$} & $79.12${\small$\pm0.35$} \\\cmidrule(lr){2-7}
 & \multirow{2}{*}{GCN} & \multirow{2}{*}{$86.98${\small$\pm0.28$}} & \multirow{2}{*}{$86.40${\small$\pm0.32$}} & 0.05 & $39.01${\small$\pm7.51$} & $45.59${\small$\pm10.81$} \\
 &  &  &  & 0.1 & $58.63${\small$\pm0.45$} & $78.86${\small$\pm0.88$} \\\cmidrule(lr){2-7}
 & \multirow{2}{*}{GIN} & \multirow{2}{*}{$81.34${\small$\pm3.84$}} & \multirow{2}{*}{$79.79${\small$\pm4.86$}} & 0.05 & $50.25${\small$\pm5.92$} & $69.15${\small$\pm4.25$} \\
 &  &  &  & 0.1 & $53.15${\small$\pm7.07$} & $71.69${\small$\pm4.18$} \\\cmidrule(lr){2-7}
 & \multirow{2}{*}{SAGE} & \multirow{2}{*}{$87.45${\small$\pm0.28$}} & \multirow{2}{*}{$86.87${\small$\pm0.33$}} & 0.05 & $3.11${\small$\pm0.82$} & $3.09${\small$\pm0.89$} \\
 &  &  &  & 0.1 & $7.56${\small$\pm1.05$} & $7.59${\small$\pm1.09$} \\ \midrule
\multirow{12}{*}{\vspace{-2em}Citeseer} & \multirow{2}{*}{S-SGC} & \multirow{2}{*}{$73.42${\small$\pm0.44$}} & \multirow{2}{*}{$69.54${\small$\pm0.40$}} & 0.05 & $50.76${\small$\pm4.26$} & $62.56${\small$\pm3.48$} \\
 &  &  &  & 0.1 & $50.82${\small$\pm4.06$} & $62.44${\small$\pm3.82$} \\\cmidrule(lr){2-7}
 & \multirow{2}{*}{S-GCN} & \multirow{2}{*}{$73.09${\small$\pm0.32$}} & \multirow{2}{*}{$69.52${\small$\pm0.23$}} & 0.05 & $52.83${\small$\pm1.79$} & $63.22${\small$\pm2.42$} \\
 &  &  &  & 0.1 & $53.50${\small$\pm1.30$} & $64.07${\small$\pm0.30$} \\\cmidrule(lr){2-7}
 & \multirow{2}{*}{SGC} & \multirow{2}{*}{$73.42${\small$\pm0.44$}} & \multirow{2}{*}{$69.54${\small$\pm0.40$}} & 0.05 & $8.88${\small$\pm3.47$} & $13.06${\small$\pm3.50$} \\
 &  &  &  & 0.1 & $30.34${\small$\pm7.32$} & $36.66${\small$\pm6.77$} \\\cmidrule(lr){2-7}
 & \multirow{2}{*}{GCN} & \multirow{2}{*}{$73.09${\small$\pm0.32$}} & \multirow{2}{*}{$69.52${\small$\pm0.23$}} & 0.05 & $10.08${\small$\pm4.19$} & $13.59${\small$\pm4.00$} \\
 &  &  &  & 0.1 & $35.06${\small$\pm6.43$} & $41.68${\small$\pm6.19$} \\\cmidrule(lr){2-7}
 & \multirow{2}{*}{GIN} & \multirow{2}{*}{$68.71${\small$\pm1.18$}} & \multirow{2}{*}{$64.82${\small$\pm2.21$}} & 0.05 & $44.49${\small$\pm11.31$} & $51.29${\small$\pm14.05$} \\
 &  &  &  & 0.1 & $48.04${\small$\pm6.37$} & $56.49${\small$\pm6.60$} \\\cmidrule(lr){2-7}
 & \multirow{2}{*}{SAGE} & \multirow{2}{*}{$74.00${\small$\pm0.36$}} & \multirow{2}{*}{$70.74${\small$\pm0.42$}} & 0.05 & $0.48${\small$\pm0.47$} & $0.63${\small$\pm0.65$} \\
 &  &  &  & 0.1 & $1.65${\small$\pm0.68$} & $1.94${\small$\pm0.85$} \\ \midrule
\multirow{12}{*}{\vspace{-2em}Pubmed} & \multirow{2}{*}{S-SGC} & \multirow{2}{*}{$83.02${\small$\pm0.48$}} & \multirow{2}{*}{$82.43${\small$\pm0.45$}} & 0.05 & $43.18${\small$\pm0.00$} & $63.44${\small$\pm0.00$} \\
 &  &  &  & 0.1 & $43.18${\small$\pm0.00$} & $63.44${\small$\pm0.00$} \\\cmidrule(lr){2-7}
 & \multirow{2}{*}{S-GCN} & \multirow{2}{*}{$83.27${\small$\pm0.62$}} & \multirow{2}{*}{$82.65${\small$\pm0.61$}} & 0.05 & $43.48${\small$\pm0.18$} & $63.68${\small$\pm0.06$} \\
 &  &  &  & 0.1 & $43.48${\small$\pm0.18$} & $63.68${\small$\pm0.06$} \\\cmidrule(lr){2-7}
 & \multirow{2}{*}{SGC} & \multirow{2}{*}{$83.02${\small$\pm0.48$}} & \multirow{2}{*}{$82.43${\small$\pm0.45$}} & 0.05 & $35.54${\small$\pm1.94$} & $51.25${\small$\pm2.64$} \\
 &  &  &  & 0.1 & $42.12${\small$\pm0.51$} & $61.45${\small$\pm0.92$} \\\cmidrule(lr){2-7}
 & \multirow{2}{*}{GCN} & \multirow{2}{*}{$83.27${\small$\pm0.62$}} & \multirow{2}{*}{$82.65${\small$\pm0.61$}} & 0.05 & $30.84${\small$\pm10.30$} & $47.40${\small$\pm10.80$} \\
 &  &  &  & 0.1 & $36.57${\small$\pm10.25$} & $55.06${\small$\pm11.62$} \\\cmidrule(lr){2-7}
 & \multirow{2}{*}{GIN} & \multirow{2}{*}{$82.82${\small$\pm4.27$}} & \multirow{2}{*}{$80.05${\small$\pm9.26$}} & 0.05 & $43.01${\small$\pm0.19$} & $61.04${\small$\pm0.13$} \\
 &  &  &  & 0.1 & $43.03${\small$\pm0.18$} & $61.08${\small$\pm0.06$} \\\cmidrule(lr){2-7}
 & \multirow{2}{*}{SAGE} & \multirow{2}{*}{$85.80${\small$\pm0.56$}} & \multirow{2}{*}{$85.50${\small$\pm0.53$}} & 0.05 & $0.88${\small$\pm0.58$} & $0.90${\small$\pm0.56$} \\
 &  &  &  & 0.1 & $2.21${\small$\pm0.71$} & $2.44${\small$\pm0.75$} \\ \bottomrule
\end{tabular}
}\label{tab:nc_architectures}
}
\end{minipage}
\begin{minipage}[t]{0.49\textwidth}
\centering
\subfloat[Decrease in accuracy for six modes of choosing $\tilde{\rmX}$ on the Cora and Citeseer datasets, compared with random injection.]{
\resizebox{\linewidth}{!}{
\begin{tabular}{llrrrr}
\toprule
\multirow[c]{2}{*}{\vspace{-3pt}Layer} & \multirow[c]{2}{*}{\vspace{-3pt}$\tilde{\rmX}$ Mode} & \multicolumn{2}{c}{Cora} & \multicolumn{2}{c}{Citeseer} \\\cmidrule(lr){3-4}\cmidrule{5-6}
&  & \multicolumn{1}{c}{PEA} & \multicolumn{1}{c}{RAND} & \multicolumn{1}{c}{PEA} & \multicolumn{1}{c}{RAND} \\

\midrule
\multirow[c]{6}{*}{SGC} & zeros & $41.58${\small$\pm3.74$} & $33.32${\small$\pm2.77$} & $8.88${\small$\pm3.47$} & $6.37${\small$\pm2.14$} \\
 & avg & $46.34${\small$\pm4.00$} & $37.28${\small$\pm2.99$} & $6.60${\small$\pm2.12$} & $5.20${\small$\pm1.78$} \\
 & randsamp & $45.66${\small$\pm2.32$} & $36.97${\small$\pm1.97$} & $7.37${\small$\pm3.71$} & $5.29${\small$\pm2.54$} \\
 & rand & \underline{$58.86${\small$\pm0.35$}} & $54.94${\small$\pm2.42$} & $15.67${\small$\pm6.05$} & $11.96${\small$\pm4.56$} \\
 & randn & $41.21${\small$\pm4.13$} & $33.06${\small$\pm2.94$} & $5.36${\small$\pm1.99$} & $4.31${\small$\pm1.53$} \\
 & randd & $43.32${\small$\pm5.30$} & $34.97${\small$\pm4.00$} & $9.27${\small$\pm5.24$} & $6.98${\small$\pm4.01$} \\
\midrule
\multirow[c]{6}{*}{GCN} & zeros & $39.01${\small$\pm7.51$} & $30.77${\small$\pm6.00$} & $10.08${\small$\pm4.19$} & $7.93${\small$\pm2.67$} \\
 & avg & $42.25${\small$\pm7.26$} & $33.60${\small$\pm6.01$} & $10.83${\small$\pm4.49$} & $8.52${\small$\pm2.89$} \\
 & randsamp & $41.65${\small$\pm8.13$} & $33.39${\small$\pm6.53$} & $13.84${\small$\pm11.10$} & $11.27${\small$\pm8.58$} \\
 & rand & \underline{$58.07${\small$\pm1.61$}} & $52.83${\small$\pm4.25$} & $20.28${\small$\pm7.26$} & $15.76${\small$\pm5.69$} \\
 & randn & $39.54${\small$\pm7.12$} & $31.07${\small$\pm5.71$} & $10.35${\small$\pm4.28$} & $7.92${\small$\pm2.79$} \\
 & randd & $44.66${\small$\pm6.63$} & $35.56${\small$\pm5.57$} & $14.32${\small$\pm8.06$} & $11.47${\small$\pm5.95$} \\
\midrule
\multirow[c]{6}{*}{GIN} & zeros & $50.25${\small$\pm5.92$} & $49.10${\small$\pm3.26$} & $44.49${\small$\pm11.31$} & $45.25${\small$\pm6.25$} \\
 & avg & $49.42${\small$\pm5.25$} & $49.09${\small$\pm3.23$} & $44.14${\small$\pm11.49$} & $45.14${\small$\pm6.44$} \\
 & randsamp & \underline{$54.78${\small$\pm0.80$}} & $48.27${\small$\pm7.22$} & $45.92${\small$\pm6.94$} & $45.17${\small$\pm10.41$} \\
 & rand & $49.76${\small$\pm7.00$} & $51.95${\small$\pm3.45$} & $45.01${\small$\pm6.75$} & $46.27${\small$\pm7.80$} \\
 & randn & $48.12${\small$\pm5.59$} & $48.52${\small$\pm4.17$} & $41.97${\small$\pm7.78$} & $44.24${\small$\pm6.79$} \\
 & randd & $52.67${\small$\pm5.75$} & $51.52${\small$\pm6.27$} & $44.49${\small$\pm8.25$} & $49.51${\small$\pm6.73$} \\
\midrule
\multirow[c]{6}{*}{SAGE} & zeros & $3.11${\small$\pm0.82$} & $3.60${\small$\pm0.71$} & $0.48${\small$\pm0.47$} & $0.89${\small$\pm0.53$} \\
 & avg & $8.08${\small$\pm1.54$} & $9.14${\small$\pm1.63$} & $1.32${\small$\pm0.72$} & $2.02${\small$\pm0.91$} \\
 & randsamp & $8.55${\small$\pm1.44$} & $9.55${\small$\pm1.75$} & $1.36${\small$\pm0.81$} & $2.22${\small$\pm0.77$} \\
 & rand & \underline{$54.52${\small$\pm2.44$}} & $52.48${\small$\pm2.90$} & $14.12${\small$\pm6.89$} & $14.03${\small$\pm6.02$} \\
 & randn & $3.18${\small$\pm0.58$} & $3.79${\small$\pm0.72$} & $0.62${\small$\pm0.51$} & $0.94${\small$\pm0.49$} \\
 & randd & $7.44${\small$\pm1.25$} & $8.73${\small$\pm1.09$} & $0.96${\small$\pm0.62$} & $1.60${\small$\pm0.87$} \\
\bottomrule
\end{tabular}
}\label{tab:heuristics_vs_xtilde_mini}
}
\vspace{6pt}
\subfloat[Clean accuracy followed by its respective decrease for two filter-based graph defense methods--Jaccard and SVD--on the Cora dataset.]{
\resizebox{\linewidth}{!}{    
\begin{tabular}{crrrr}
\toprule
\multicolumn{1}{c}{Defense$\rightarrow$} & \multicolumn{1}{c}{Jaccard (0.01)}       & \multicolumn{1}{c}{Jaccard (0.05)}       & \multicolumn{1}{c}{SVD (20)}             & \multicolumn{1}{c}{SVD (200)}            \\
\multicolumn{1}{c}{(Clean Acc)}                                      & \multicolumn{1}{c}{(86.08±0.47)}         & \multicolumn{1}{c}{(83.49±0.47)}         & \multicolumn{1}{c}{(75.76±0.49)}         & \multicolumn{1}{c}{(83.71±0.45)}         \\ \cmidrule(lr){1-1}\cmidrule(lr){2-5}
\multicolumn{1}{c}{$r$}                                   & Accuracy$\downarrow$ & Accuracy$\downarrow$ & Accuracy$\downarrow$ & Accuracy$\downarrow$ \\ \midrule
0.01                                                      & 0.07±0.36                                & 0.00±0.23                                & 11.73±1.69                               & 5.00±0.78                                \\
0.05                                                      & 7.06±1.70                                & 0.37±0.31                                & 34.88±4.14                               & 35.27±4.19                               \\
0.10                                                       & 37.78±4.42                               & 3.70±1.12                                & 46.98±0.90                               & 54.91±0.55                               \\ \bottomrule
\end{tabular}
}\label{tab:filter_defense_cora_no_f1}
}
\end{minipage}
\caption{Base performance across different architectures (Left); Performance of different $\tilde{\rmX}$ modes across layers for a fixed budget (Top-right); and Performance of \MethodName vs filter-based defense methods (Bottom-right).}
\end{figure*}

%% file: tables/bigtable_gcn_vs_baselines.tex
\begin{table*}[htb]
\caption{Results of \MethodName and the baselines on Node c-lassification using the \GCN architecture as the defending model. The accuracy and (macro) F1 score on the clean datasets are reported first, followed by their respective \emph{decrease} for each method.}
\resizebox{\textwidth}{!}{
\begin{tabular}{lrrrrrrrrrr}
\toprule
\multirow{2}{*}{Dataset}  & \multicolumn{1}{c}{\multirow{2}{*}{$r$}} & \multicolumn{1}{c}{\multirow{2}{*}{$\nv$}}                                                        & \multicolumn{2}{c}{AGIA}                                                        & \multicolumn{2}{c}{TDGIA}                                                       & \multicolumn{2}{c}{ATDGIA}                                                      & \multicolumn{2}{c}{\MethodName}                                                         \\ \cmidrule(lr){4-5}\cmidrule(lr){6-7}\cmidrule(lr){8-9}\cmidrule(lr){10-11} 
                          & \multicolumn{1}{c}{}                     & \multicolumn{1}{c}{}                                                           & \multicolumn{1}{c}{$\rm Accuracy_\downarrow$} & \multicolumn{1}{c}{$\rm F1_\downarrow$} & \multicolumn{1}{c}{$\rm Accuracy_\downarrow$} & \multicolumn{1}{c}{$\rm F1_\downarrow$} & \multicolumn{1}{c}{$\rm Accuracy_\downarrow$} & \multicolumn{1}{c}{$\rm F1_\downarrow$} & \multicolumn{1}{c}{$\rm Accuracy_\downarrow$} & \multicolumn{1}{c}{$\rm F1_\downarrow$} \\ \midrule
\multirow{4}{*}{Cora}     & 0.001                                    & 2                                          & $0.04${\scriptsize$\pm0.11$}                   & $0.03${\scriptsize$\pm0.08$}             & $0.07${\scriptsize$\pm0.15$}                   & $0.05${\scriptsize$\pm0.10$}             & \underline{$0.07${\scriptsize$\pm0.15$}}                   & \underline{$0.07${\scriptsize$\pm0.15$}}             & $0.07${\scriptsize$\pm0.17$}                   & $0.04${\scriptsize$\pm0.17$}             \\
                          & 0.01                                     & 27                                         & $2.52${\scriptsize$\pm0.62$}                   & $2.21${\scriptsize$\pm0.62$}             & $2.41${\scriptsize$\pm0.48$}                   & $2.04${\scriptsize$\pm0.56$}             & \underline{$3.78${\scriptsize$\pm0.92$}}          & \underline{$5.11${\scriptsize$\pm1.69$}}    & $0.72${\scriptsize$\pm0.42$}                   & $0.57${\scriptsize$\pm0.42$}             \\
                          & 0.05                                     & 135                                        & $8.96${\scriptsize$\pm1.18$}                   & $10.24${\scriptsize$\pm1.58$}            & $9.81${\scriptsize$\pm1.55$}                   & $11.05${\scriptsize$\pm1.89$}            & $12.78${\scriptsize$\pm2.25$}                  & $13.55${\scriptsize$\pm2.24$}            & \underline{$36.12${\scriptsize$\pm1.94$}}         & \underline{$45.34${\scriptsize$\pm3.18$}}   \\
                          & 0.1                                      & 270                                        & $12.44${\scriptsize$\pm1.30$}                  & $13.38${\scriptsize$\pm1.69$}            & $13.00${\scriptsize$\pm1.45$}                  & $13.72${\scriptsize$\pm1.88$}            & $14.52${\scriptsize$\pm2.00$}                  & $15.05${\scriptsize$\pm2.35$}            & \underline{$51.65${\scriptsize$\pm0.77$}}         & \underline{$73.08${\scriptsize$\pm0.77$}}   \\\midrule
\multirow{4}{*}{Citeseer} & 0.001                                    & 3                                          & $0.15${\scriptsize$\pm0.15$}                   & $0.11${\scriptsize$\pm0.11$}             & $0.15${\scriptsize$\pm0.15$}                   & $0.11${\scriptsize$\pm0.11$}             & \underline{$0.18${\scriptsize$\pm0.15$}}          & \underline{$0.17${\scriptsize$\pm0.13$}}    & $0.00${\scriptsize$\pm0.07$}                   & $-0.00${\scriptsize$\pm0.06$}            \\
                          & 0.01                                     & 33                                         & $2.26${\scriptsize$\pm0.28$}                   & $2.03${\scriptsize$\pm0.24$}             & \underline{$2.32${\scriptsize$\pm0.52$}}          & \underline{$2.08${\scriptsize$\pm0.48$}}    & $1.99${\scriptsize$\pm0.63$}                   & $1.92${\scriptsize$\pm0.58$}             & $0.26${\scriptsize$\pm0.52$}                   & $0.45${\scriptsize$\pm0.70$}             \\
                          & 0.05                                     & 166                                        & $8.61${\scriptsize$\pm0.45$}                   & $7.68${\scriptsize$\pm0.44$}             & $8.49${\scriptsize$\pm0.87$}                   & $7.58${\scriptsize$\pm0.73$}             & $12.92${\scriptsize$\pm0.68$}                  & $11.23${\scriptsize$\pm0.59$}            & \underline{$15.16${\scriptsize$\pm6.65$}}         & \underline{$17.97${\scriptsize$\pm6.46$}}   \\
                          & 0.1                                      & 332                                        & $11.05${\scriptsize$\pm0.82$}                  & $9.66${\scriptsize$\pm0.74$}             & $10.81${\scriptsize$\pm0.61$}                  & $9.55${\scriptsize$\pm0.57$}             & $14.61${\scriptsize$\pm1.44$}                  & $12.62${\scriptsize$\pm1.22$}            & \underline{$37.64${\scriptsize$\pm7.82$}}         & \underline{$45.38${\scriptsize$\pm8.29$}}   \\\midrule
\multirow{4}{*}{Pubmed}   & 0.001                                    & 19                                         & $0.04${\scriptsize$\pm0.02$}                   & $0.03${\scriptsize$\pm0.03$}             & $0.02${\scriptsize$\pm0.02$}                   & $0.03${\scriptsize$\pm0.03$}             & $0.03${\scriptsize$\pm0.09$}                   & $0.02${\scriptsize$\pm0.11$}             & \underline{$0.21${\scriptsize$\pm0.14$}}          & \underline{$0.33${\scriptsize$\pm0.17$}}    \\
                          & 0.01                                     & 197                                        & $0.76${\scriptsize$\pm0.26$}                   & $0.76${\scriptsize$\pm0.26$}             & $0.72${\scriptsize$\pm0.26$}                   & $0.72${\scriptsize$\pm0.31$}             & $1.10${\scriptsize$\pm0.10$}                   & $1.16${\scriptsize$\pm0.15$}             & \underline{$4.33${\scriptsize$\pm0.76$}}          & \underline{$7.24${\scriptsize$\pm1.48$}}    \\
                          & 0.05                                     & 985                                        & $3.99${\scriptsize$\pm0.37$}                   & $4.14${\scriptsize$\pm0.41$}             & $3.78${\scriptsize$\pm0.51$}                   & $3.88${\scriptsize$\pm0.51$}             & $10.22${\scriptsize$\pm1.08$}                  & $10.65${\scriptsize$\pm1.11$}            & \underline{$34.24${\scriptsize$\pm4.64$}}         & \underline{$50.24${\scriptsize$\pm5.42$}}   \\
                          & 0.1                                      & 1971                                       & $11.98${\scriptsize$\pm1.14$}                  & $12.39${\scriptsize$\pm1.30$}            & $11.54${\scriptsize$\pm0.97$}                  & $11.88${\scriptsize$\pm1.19$}            & $11.98${\scriptsize$\pm1.13$}                  & $12.29${\scriptsize$\pm1.31$}            & \underline{$41.78${\scriptsize$\pm1.56$}}         & \underline{$61.09${\scriptsize$\pm2.35$}}   \\ \bottomrule
\end{tabular}
}
\label{tab:GCN_baselines}
\end{table*}

%% file: tables/heuristics_vs_xtilde.tex
\begin{table*}[h]
\centering
\caption{Performance of different ways of choosing $\tilde{\rmX}$ across different MPNN architectures for the Cora and Citeseer datasets, with a fixed budget of $r=0.05$. For a given $\tilde{\rmX}$, the maximum decrease across all methods is underlined.}
\begin{tabular}{lclrrrrr}
\toprule
\multirow[c]{2}{*}{\vspace{-3pt}Dataset} & \multirow[c]{2}{*}{\vspace{-5pt}\shortstack{Layer\\(Clean Acc)}} & \multirow[c]{2}{*}{\vspace{-3pt}$\tilde{\rmX}$ Mode} & \multicolumn{5}{c}{Accuracy Decrease (in \%)} \\\cmidrule{4-8}
&  &  & \multicolumn{1}{c}{PEA} & \multicolumn{1}{c}{RAND} & \multicolumn{1}{c}{DEG} & \multicolumn{1}{c}{BET} & \multicolumn{1}{c}{CENT} \\
\midrule
\multirow[c]{24}{*}{\vspace{-1.5em}Cora} & \multirow[c]{6}{*}{\shortstack{SGC\\[0.2em]($86.95${\scriptsize$\pm0.27$})}} & zeros & \underline{$41.58${\scriptsize$\pm3.74$}} & $33.32${\scriptsize$\pm2.77$} & $10.63${\scriptsize$\pm2.23$} & $1.03${\scriptsize$\pm0.26$} & $-0.01${\scriptsize$\pm0.32$} \\
 &  & avg & \underline{$46.34${\scriptsize$\pm4.00$}} & $37.28${\scriptsize$\pm2.99$} & $16.89${\scriptsize$\pm3.13$} & $1.52${\scriptsize$\pm0.32$} & $0.04${\scriptsize$\pm0.33$} \\
 &  & randsamp & \underline{$45.66${\scriptsize$\pm2.32$}} & $36.97${\scriptsize$\pm1.97$} & $16.59${\scriptsize$\pm1.79$} & $1.60${\scriptsize$\pm0.26$} & $0.03${\scriptsize$\pm0.38$} \\
 &  & rand & \underline{$58.86${\scriptsize$\pm0.35$}} & $54.94${\scriptsize$\pm2.42$} & $50.81${\scriptsize$\pm3.69$} & $6.61${\scriptsize$\pm0.75$} & $2.92${\scriptsize$\pm0.84$} \\
 &  & randn & \underline{$41.21${\scriptsize$\pm4.13$}} & $33.06${\scriptsize$\pm2.94$} & $10.34${\scriptsize$\pm2.24$} & $1.09${\scriptsize$\pm0.20$} & $-0.03${\scriptsize$\pm0.20$} \\
 &  & randd & \underline{$43.32${\scriptsize$\pm5.30$}} & $34.97${\scriptsize$\pm4.00$} & $14.41${\scriptsize$\pm4.03$} & $1.29${\scriptsize$\pm0.25$} & $0.05${\scriptsize$\pm0.26$} \\
\cmidrule(lr){2-8}
 & \multirow[c]{6}{*}{\shortstack{GCN\\[0.2em]($86.96${\scriptsize$\pm0.28$})}} & zeros & \underline{$39.01${\scriptsize$\pm7.51$}} & $30.77${\scriptsize$\pm6.00$} & $10.18${\scriptsize$\pm4.79$} & $0.96${\scriptsize$\pm0.34$} & $0.02${\scriptsize$\pm0.33$} \\
 &  & avg & \underline{$42.25${\scriptsize$\pm7.26$}} & $33.60${\scriptsize$\pm6.01$} & $14.19${\scriptsize$\pm5.84$} & $1.30${\scriptsize$\pm0.45$} & $0.05${\scriptsize$\pm0.33$} \\
 &  & randsamp & \underline{$41.65${\scriptsize$\pm8.13$}} & $33.39${\scriptsize$\pm6.53$} & $14.41${\scriptsize$\pm4.42$} & $1.27${\scriptsize$\pm0.38$} & $0.02${\scriptsize$\pm0.23$} \\
 &  & rand & \underline{$58.07${\scriptsize$\pm1.61$}} & $52.83${\scriptsize$\pm4.25$} & $47.66${\scriptsize$\pm6.68$} & $6.97${\scriptsize$\pm1.29$} & $3.05${\scriptsize$\pm1.02$} \\
 &  & randn & \underline{$39.54${\scriptsize$\pm7.12$}} & $31.07${\scriptsize$\pm5.71$} & $10.04${\scriptsize$\pm4.09$} & $1.01${\scriptsize$\pm0.38$} & $0.10${\scriptsize$\pm0.30$} \\
 &  & randd & \underline{$44.66${\scriptsize$\pm6.63$}} & $35.56${\scriptsize$\pm5.57$} & $16.36${\scriptsize$\pm5.58$} & $1.45${\scriptsize$\pm0.45$} & $-0.07${\scriptsize$\pm0.28$} \\
\cmidrule(lr){2-8}
 & \multirow[c]{6}{*}{\shortstack{GIN\\[0.2em]($82.19${\scriptsize$\pm3.45$})}} & zeros & \underline{$50.25${\scriptsize$\pm5.92$}} & $49.10${\scriptsize$\pm3.26$} & $43.09${\scriptsize$\pm5.58$} & $2.22${\scriptsize$\pm3.75$} & $-0.13${\scriptsize$\pm4.06$} \\
 &  & avg & \underline{$49.42${\scriptsize$\pm5.25$}} & $49.09${\scriptsize$\pm3.23$} & $43.85${\scriptsize$\pm5.45$} & $5.75${\scriptsize$\pm4.52$} & $0.25${\scriptsize$\pm4.54$} \\
 &  & randsamp & \underline{$54.78${\scriptsize$\pm0.80$}} & $48.27${\scriptsize$\pm7.22$} & $42.10${\scriptsize$\pm10.34$} & $5.77${\scriptsize$\pm3.65$} & $0.25${\scriptsize$\pm1.24$} \\
 &  & rand & $49.76${\scriptsize$\pm7.00$} & \underline{$51.95${\scriptsize$\pm3.45$}} & $50.79${\scriptsize$\pm3.73$} & $25.59${\scriptsize$\pm5.74$} & $14.15${\scriptsize$\pm1.86$} \\
 &  & randn & $48.12${\scriptsize$\pm5.59$} & \underline{$48.52${\scriptsize$\pm4.17$}} & $40.58${\scriptsize$\pm5.71$} & $2.25${\scriptsize$\pm4.10$} & $0.01${\scriptsize$\pm4.23$} \\
 &  & randd & \underline{$52.67${\scriptsize$\pm5.75$}} & $51.52${\scriptsize$\pm6.27$} & $43.90${\scriptsize$\pm10.07$} & $6.48${\scriptsize$\pm4.85$} & $0.49${\scriptsize$\pm1.02$} \\
\cmidrule(lr){2-8}
 & \multirow[c]{6}{*}{\shortstack{SAGE\\[0.2em]($87.42${\scriptsize$\pm0.31$})}} & zeros & $3.11${\scriptsize$\pm0.82$} & \underline{$3.60${\scriptsize$\pm0.71$}} & $1.87${\scriptsize$\pm0.39$} & $0.28${\scriptsize$\pm0.30$} & $0.13${\scriptsize$\pm0.31$} \\
 &  & avg & $8.08${\scriptsize$\pm1.54$} & \underline{$9.14${\scriptsize$\pm1.63$}} & $4.96${\scriptsize$\pm1.18$} & $0.60${\scriptsize$\pm0.50$} & $0.49${\scriptsize$\pm0.34$} \\
 &  & randsamp & $8.55${\scriptsize$\pm1.44$} & \underline{$9.55${\scriptsize$\pm1.75$}} & $5.43${\scriptsize$\pm1.37$} & $0.53${\scriptsize$\pm0.37$} & $0.49${\scriptsize$\pm0.43$} \\
 &  & rand & \underline{$54.52${\scriptsize$\pm2.44$}} & $52.48${\scriptsize$\pm2.90$} & $49.11${\scriptsize$\pm4.74$} & $7.02${\scriptsize$\pm1.13$} & $10.49${\scriptsize$\pm1.23$} \\
 &  & randn & $3.18${\scriptsize$\pm0.58$} & \underline{$3.79${\scriptsize$\pm0.72$}} & $1.99${\scriptsize$\pm0.46$} & $0.27${\scriptsize$\pm0.31$} & $0.16${\scriptsize$\pm0.24$} \\
 &  & randd & $7.44${\scriptsize$\pm1.25$} & \underline{$8.73${\scriptsize$\pm1.09$}} & $4.70${\scriptsize$\pm1.08$} & $0.47${\scriptsize$\pm0.43$} & $0.46${\scriptsize$\pm0.42$} \\
\midrule
\multirow[c]{24}{*}{\vspace{-1.5em}Citeseer} & \multirow[c]{6}{*}{\shortstack{SGC\\[0.2em]($73.08${\scriptsize$\pm0.37$})}} & zeros & \underline{$8.88${\scriptsize$\pm3.47$}} & $6.37${\scriptsize$\pm2.14$} & $1.25${\scriptsize$\pm0.44$} & $0.15${\scriptsize$\pm0.47$} & $0.10${\scriptsize$\pm0.45$} \\
 &  & avg & \underline{$6.60${\scriptsize$\pm2.12$}} & $5.20${\scriptsize$\pm1.78$} & $1.14${\scriptsize$\pm0.56$} & $0.66${\scriptsize$\pm0.35$} & $-0.17${\scriptsize$\pm0.28$} \\
 &  & randsamp & \underline{$7.37${\scriptsize$\pm3.71$}} & $5.29${\scriptsize$\pm2.54$} & $1.10${\scriptsize$\pm0.97$} & $0.57${\scriptsize$\pm0.28$} & $-0.22${\scriptsize$\pm0.24$} \\
 &  & rand & \underline{$15.67${\scriptsize$\pm6.05$}} & $11.96${\scriptsize$\pm4.56$} & $6.77${\scriptsize$\pm3.15$} & $1.35${\scriptsize$\pm0.61$} & $-0.19${\scriptsize$\pm0.30$} \\
 &  & randn & \underline{$5.36${\scriptsize$\pm1.99$}} & $4.31${\scriptsize$\pm1.53$} & $0.93${\scriptsize$\pm0.62$} & $0.63${\scriptsize$\pm0.40$} & $-0.16${\scriptsize$\pm0.30$} \\
 &  & randd & \underline{$9.27${\scriptsize$\pm5.24$}} & $6.98${\scriptsize$\pm4.01$} & $1.43${\scriptsize$\pm1.39$} & $0.66${\scriptsize$\pm0.48$} & $-0.22${\scriptsize$\pm0.40$} \\
\cmidrule(lr){2-8}
 & \multirow[c]{6}{*}{\shortstack{GCN\\[0.2em]($73.13${\scriptsize$\pm0.36$})}} & zeros & \underline{$10.08${\scriptsize$\pm4.19$}} & $7.93${\scriptsize$\pm2.67$} & $1.89${\scriptsize$\pm0.96$} & $0.45${\scriptsize$\pm0.31$} & $-0.19${\scriptsize$\pm0.33$} \\
 &  & avg & \underline{$10.83${\scriptsize$\pm4.49$}} & $8.52${\scriptsize$\pm2.89$} & $2.25${\scriptsize$\pm1.18$} & $0.52${\scriptsize$\pm0.32$} & $-0.20${\scriptsize$\pm0.32$} \\
 &  & randsamp & \underline{$13.84${\scriptsize$\pm11.10$}} & $11.27${\scriptsize$\pm8.58$} & $3.88${\scriptsize$\pm3.73$} & $0.67${\scriptsize$\pm0.48$} & $-0.17${\scriptsize$\pm0.40$} \\
 &  & rand & \underline{$20.28${\scriptsize$\pm7.26$}} & $15.76${\scriptsize$\pm5.69$} & $8.38${\scriptsize$\pm4.04$} & $1.24${\scriptsize$\pm0.66$} & $-0.17${\scriptsize$\pm0.41$} \\
 &  & randn & \underline{$10.35${\scriptsize$\pm4.28$}} & $7.92${\scriptsize$\pm2.79$} & $1.82${\scriptsize$\pm0.92$} & $0.49${\scriptsize$\pm0.39$} & $-0.19${\scriptsize$\pm0.47$} \\
 &  & randd & \underline{$14.32${\scriptsize$\pm8.06$}} & $11.47${\scriptsize$\pm5.95$} & $3.82${\scriptsize$\pm2.54$} & $0.75${\scriptsize$\pm0.36$} & $-0.22${\scriptsize$\pm0.23$} \\
\cmidrule(lr){2-8}
 & \multirow[c]{6}{*}{\shortstack{GIN\\[0.2em]($68.96${\scriptsize$\pm1.05$})}} & zeros & $44.49${\scriptsize$\pm11.31$} & \underline{$45.25${\scriptsize$\pm6.25$}} & $40.67${\scriptsize$\pm11.70$} & $2.98${\scriptsize$\pm2.01$} & $-0.55${\scriptsize$\pm1.19$} \\
 &  & avg & $44.14${\scriptsize$\pm11.49$} & \underline{$45.14${\scriptsize$\pm6.44$}} & $40.66${\scriptsize$\pm11.74$} & $4.86${\scriptsize$\pm2.91$} & $-0.06${\scriptsize$\pm1.58$} \\
 &  & randsamp & \underline{$45.92${\scriptsize$\pm6.94$}} & $45.17${\scriptsize$\pm10.41$} & $42.36${\scriptsize$\pm12.81$} & $4.46${\scriptsize$\pm2.31$} & $-0.23${\scriptsize$\pm0.45$} \\
 &  & rand & $45.01${\scriptsize$\pm6.75$} & \underline{$46.27${\scriptsize$\pm7.80$}} & $43.39${\scriptsize$\pm11.69$} & $12.58${\scriptsize$\pm5.53$} & $2.95${\scriptsize$\pm4.13$} \\
 &  & randn & $41.97${\scriptsize$\pm7.78$} & \underline{$44.24${\scriptsize$\pm6.79$}} & $41.79${\scriptsize$\pm9.63$} & $2.75${\scriptsize$\pm1.81$} & $-0.57${\scriptsize$\pm0.90$} \\
 &  & randd & $44.49${\scriptsize$\pm8.25$} & \underline{$49.51${\scriptsize$\pm6.73$}} & $45.87${\scriptsize$\pm8.89$} & $4.92${\scriptsize$\pm3.97$} & $-0.41${\scriptsize$\pm0.94$} \\
\cmidrule(lr){2-8}
 & \multirow[c]{6}{*}{\shortstack{SAGE\\[0.2em]($73.98${\scriptsize$\pm0.40$})}} & zeros & $0.48${\scriptsize$\pm0.47$} & \underline{$0.89${\scriptsize$\pm0.53$}} & $0.37${\scriptsize$\pm0.35$} & $-0.08${\scriptsize$\pm0.29$} & $-0.16${\scriptsize$\pm0.39$} \\
 &  & avg & $1.32${\scriptsize$\pm0.72$} & \underline{$2.02${\scriptsize$\pm0.91$}} & $0.98${\scriptsize$\pm0.51$} & $-0.07${\scriptsize$\pm0.27$} & $-0.16${\scriptsize$\pm0.39$} \\
 &  & randsamp & $1.36${\scriptsize$\pm0.81$} & \underline{$2.22${\scriptsize$\pm0.77$}} & $1.16${\scriptsize$\pm0.55$} & $0.01${\scriptsize$\pm0.34$} & $-0.10${\scriptsize$\pm0.46$} \\
 &  & rand & \underline{$14.12${\scriptsize$\pm6.89$}} & $14.03${\scriptsize$\pm6.02$} & $9.19${\scriptsize$\pm5.32$} & $0.94${\scriptsize$\pm0.67$} & $-0.16${\scriptsize$\pm0.44$} \\
 &  & randn & $0.62${\scriptsize$\pm0.51$} & \underline{$0.94${\scriptsize$\pm0.49$}} & $0.35${\scriptsize$\pm0.48$} & $-0.08${\scriptsize$\pm0.36$} & $-0.12${\scriptsize$\pm0.46$} \\
 &  & randd & $0.96${\scriptsize$\pm0.62$} & \underline{$1.60${\scriptsize$\pm0.87$}} & $0.74${\scriptsize$\pm0.61$} & $0.04${\scriptsize$\pm0.47$} & $-0.11${\scriptsize$\pm0.47$} \\
\bottomrule
\end{tabular}
\label{tab:heuristics_vs_xtilde_full}
\end{table*}